\pdfoutput=1
\documentclass[11pt]{article}

\usepackage[final]{acl}

\usepackage{times}
\usepackage{latexsym}

\usepackage[T1]{fontenc}
\usepackage[utf8]{inputenc}

\usepackage{microtype}

\usepackage{inconsolata}

\usepackage{graphicx}

\usepackage{amsfonts}       % blackboard math symbols
\usepackage{xcolor}         % colors
\definecolor{darkbrown}{rgb}{0.5, 0.25, 0.20}
\definecolor{dimeblue}{rgb}{0, 0.3, 0.7}
\usepackage{booktabs}       % professional-quality tables

\usepackage{algorithm}
\usepackage{algorithmic}
\usepackage{amsfonts,amsmath,amssymb,amsthm, pbox, bbm}
\usepackage{graphicx}
\usepackage{tikz}           % For creating graphics programmatically
\usepackage{subcaption}
\usepackage{wrapfig}
\usepackage{lipsum}
\usepackage{dsfont}
\usepackage{colortbl}
\usepackage{pgfplots}
\usepackage{hyperref}       % For code link
\usepackage{diagbox}        % For block out experiment results
\usepackage{pifont}
\usepackage{cancel}

\theoremstyle{definition}

\theoremstyle{plain}
\newtheorem{theorem}{Theorem}
\newtheorem{corollary}{Corollary}

\theoremstyle{remark}

\newcommand{\ty}[1]{}
\newcommand{\ry}[1]{}

\title{SpecHub: Provable Acceleration to Multi-Draft Speculative Decoding}
\author{
  \textbf{Ryan Sun\textsuperscript{1,2}},
  \textbf{Tianyi Zhou\textsuperscript{3}},
  \textbf{Xun Chen\textsuperscript{2}},
  \textbf{Lichao Sun\textsuperscript{1}},
\\
  \textsuperscript{1}Lehigh University,
  \textsuperscript{2}Samsung Research America,
  \textsuperscript{3}University of Maryland, College park
\\
  \small{
    \textbf{Correspondence:} \href{mailto:email@domain}{lis221@lehigh.edu}
  }
}

\begin{document}
\maketitle
\begin{abstract}

Large Language Models (LLMs) have become essential in advancing natural language processing (NLP) tasks, but their sequential token generation limits inference speed. 
Multi-Draft Speculative Decoding (MDSD) offers a promising solution by using a smaller draft model to generate multiple token sequences, which the target LLM verifies in parallel.
However, current heuristic approaches, such as Recursive Rejection Sampling (RRS), suffer from low acceptance rates in subsequent drafts, limiting the advantages of using multiple drafts. 
Meanwhile, Optimal Transport with Membership Cost (OTM) can theoretically improve acceptance rates, but its computational cost is too high for real-time use.
We present SpecHub, a novel, efficient sampling-verification method for MDSD that improves acceptance rates with only linear computational overhead. By simplifying the OTM problem into a compact Linear Programming model, SpecHub significantly reduces computational complexity. It further accelerates sampling by leveraging a sparse joint distribution, focusing computation on high-probability token sequences.
%It integrates seamlessly into existing MDSD frameworks.
In extensive experiments, Spechub consistently generates 0.05-0.27 and 0.02-0.16 more tokens per step than RRS and RRS without replacement. \looseness-1 
We attach our code at \url{https://github.com/MasterGodzilla/Speculative_decoding_OT}. 
\end{abstract}

\section{Introduction}

%%%%%%%%%%%%%%%%%%%%%%%% Add a figure of what speculative decoding, and a token tree looks like, so people can understand. %%%%%%%%%%%%%%%%%%%%

%%%%%%%%%%%%%%%%%% However, no so many paragraphs on SD and Multi-draft %%%%%%%%%%%%%%%%%%%%%%%%%%%%%%%

%%%%%%%%%%%%%     Solution: only explain the part related to motivations
% - less on GPU
% - quickly make people focus on the motivations of the method，口袋迅速收紧
% - 前三段没不够引出motiviation。要有很强的indicator指向问题。痛点突出。
% - 细节太多，重点不够

% 格式：
% intro - challenge - limitations

%%%%%%%%%%%%%%%%%%%%%

%%%%%%%%%%%%%    Goal: first three paragraphs gives motivations %%%%%%%%%

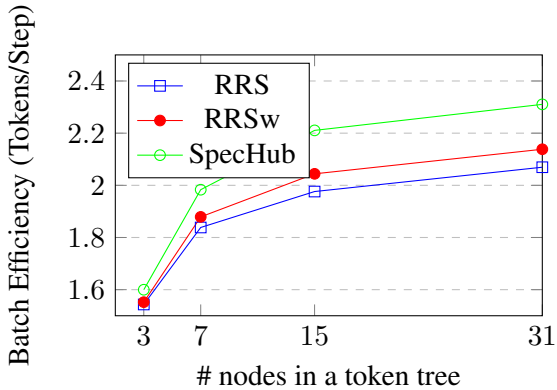
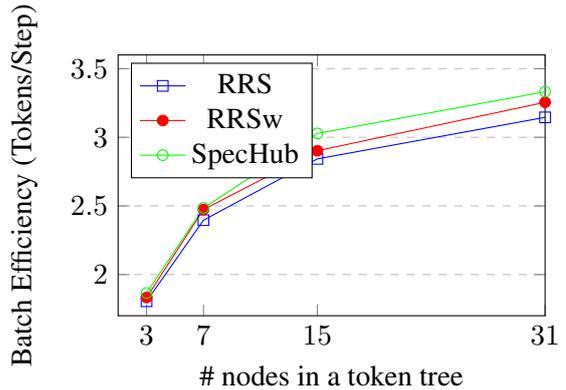
\begin{figure*}[t]
    \centering
    \begin{subfigure}[b]{0.45\textwidth}
        \centering
        \begin{tikzpicture}
            \begin{axis}[
                xlabel=\# nodes in a token tree,
                ylabel=Batch Efficiency (Tokens/Step),
                xmin=1, xmax=31, 
                ymin=1.5, ymax=2.5, % Adjust y-axis limits as needed
                xtick={3,7,15,31},
                legend pos=north west,
                ymajorgrids=true,
                grid style=dashed,
                width=\textwidth, % Adjust width as needed
                height=0.7\textwidth, % Adjust height as needed
            ]
            
            \addplot[
                color=blue,
                mark=square,
            ]
                coordinates {
                    (3, 1.5432) (7, 1.8384) (15, 1.9762) (31, 2.0694) 
                };
                \addlegendentry{RRS}
               \addplot[
                color=red,
                mark=*,
            ]
                coordinates {
                    (3, 1.5521) (7, 1.8790) (15, 2.0441) (31, 2.1383) 
                };
                \addlegendentry{RRSw}

            \addplot[
                color=green,
                mark=o,
            ]
                coordinates {
                    (3, 1.5997) (7, 1.9832) (15, 2.2106) (31, 2.3104) 
                };
                \addlegendentry{SpecHub}
                
            \end{axis}
        \end{tikzpicture}
        \caption{Llama2-7B with JF68m draft model on CNN dataset.}
        \label{subfig:llama2_efficiency}
    \end{subfigure}
    \hfill
    \begin{subfigure}[b]{0.45\textwidth}
        \centering
        \begin{tikzpicture}
            \begin{axis}[
                xlabel=\# nodes in a token tree,
                ylabel=Batch Efficiency (Tokens/Step),
                xmin=1, xmax=31, 
                ymin=1.7, ymax=3.6, % Adjust y-axis limits as needed
                xtick={3,7,15,31},
                legend pos=north west,
                ymajorgrids=true,
                grid style=dashed,
                width=\textwidth, % Adjust width as needed
                height=0.7\textwidth, % Adjust height as needed
            ]
            
            \addplot[
                color=blue,
                mark=square,
            ]
                coordinates {
                    (3, 1.8054) (7, 2.3961) (15, 2.8425) (31, 3.1451)
                };
                \addlegendentry{RRS}
            
            \addplot[
                color=red,
                mark=*,
            ]
                coordinates {
                    (3, 1.8327) (7, 2.4737) (15, 2.9019) (31, 3.2548)
                };
                \addlegendentry{RRSw}

            \addplot[
                color=green,
                mark=o,
            ]
                coordinates {
                    (3, 1.8655) (7, 2.4850) (15, 3.0281) (31, 3.3318)
                };
                \addlegendentry{SpecHub}
                
            \end{axis}
        \end{tikzpicture}
        \caption{Vicuna-7B with EAGLE draft model on MT-Bench.}
        \label{subfig:vicuna_efficiency}
    \end{subfigure}
    \caption{Batch efficiency of SpecHub, RRS, and RRSw with different numbers of nodes in a binary token tree with temperature $T=1.0$.}
    \label{fig:batch_efficiency_comparison}
\end{figure*}
%%%% Introduce background 

With the growing adoption of Large Language Models (LLMs) in diverse applications, there is a significant demand for faster inference and lower latency in both local computing and online API services. However, the sequential generation process of autoregressive language models complicates parallel computation. This challenge is exacerbated by the memory limitations of current hardware architectures, where RAM and cache communication latencies often constrain performance, resulting in underutilized computing capacity.

Speculative decoding \citep{leviathan2023fast,chen2023accelerating} accelerates LLM inference while preserving the model's output distribution. By generating a sequence of draft tokens in advance using a smaller model, it leverages GPUs to verify tokens simultaneously through rejection sampling. Recent advancements \citep{chen2024sequoia,jeon2024recursive,sun2024spectr,miao2023specinfer} have further enhanced this approach by introducing tree-structured multi-drafts, where each path represents a draft. These tokens are verified in parallel during a single forward pass of the LLM. Using a token tree increases the number of accepted tokens by providing multiple options for each token position, thus increasing the overall acceptance rate of the algorithm and generation efficiency. 

%Despite the variety in tree constructions, draft model designs, and hardware optimizations, all of those multi-draft methods rely on recursive rejection sampling (RRS) for the acceptance step. Despite optimality in the first draft acceptance rate where the draft and target distributions align well, the proceeding iterations require models to accept according to the residual distribution, which shares no similarity to the draft distribution used for the sampling step. As a result, the proceeding runs suffer from a dramatically lower acceptance rate \cite{chen2023cascade}. The problem is slightly alleviated with the introduction of sampling-without-replacement \cite{chen2024sequoia,jeon2024recursive}, but still far from optimal. 
%Despite various tree constructions, draft model designs, and hardware optimizations, existing multi-draft methods depend on recursive rejection sampling (RRS) for acceptance. While RRS is optimal for the first draft when draft and target distributions align, subsequent iterations must accept tokens based on a residual distribution, which diverges from the draft distribution. 

%However, the current sampling and verification algorithms is far from optimal. Although relies on Recursive Rejection Sampling, which 

Despite having various tree constructions, draft model designs, and hardware optimizations, existing Multi-Draft Speculative Decoding (MDSD) methods depend on recursive rejection sampling (RRS) for acceptance, which is far from optimal. \ty{Here is a place to introduce the undersampling and oversampling issues.} While RRS greedily accepts the token from the first draft, it does not consider the subsequent drafts and misses the opportunity to dynamically adjust the current token's acceptance strategy to improve the acceptance rates of the later drafts.
RRS prioritizes the first draft's tokens but fails to dynamically adjust acceptance strategies for subsequent drafts. As a result, later iterations use a residual distribution modified by previous acceptances, leading to misalignment with the original draft distribution and lower acceptance rates~\citep{chen2023cascade}. 
\citet{sun2024spectr} shows that the acceptance rule could be optimized through an Optimal Transport problem with Membership Cost (OTM), which maximizes acceptance rates by better aligning draft tokens with the accepted token. However, OTM requires tremendous computation overhead and is not practically feasible. 
\ry{RRSw is not important enough to make into the Intro, thus deleted}

In this paper, we address the trade-off between computational efficiency and sampling optimality 
in Multi-Draft Speculative Decoding (MDSD). We first reduce the OTM formulation to a much smaller linear programming (LP) by focusing only on the transport plan of scenarios where at least one draft gets accepted. 
We then investigate the overlooked design choice of draft sampling. While all previous methods used either sampling with or without replacement, which makes finding the optimal solution notoriously hard, we show that an optimal acceptance rule can be trivially obtained if we instead choose only certain drafts of tokens. As a result, we can develop practical algorithms that balance acceptance rate with computation overhead. 

Building on the new LP formulation and insights, we introduce SpecHub, a faster sampling-verification paradigm with only linear computational overhead. Instead of constructing a dense distribution of $k$-draft and the accepted token, SpecHub strategically selects drafts containing the highest probability token sampled from the draft model.
The top draft token serves as a transport hub for an oversampled token~\footnote{Draft model probability exceeds that of the target model.} to transfer its excessive probability mass to an undersampled token. 
This sparse structure simplifies and accelerates the underlying linear programming. 
SpecHub performs particularly well on LLMs whose output distributions are concentrated on the top token, resulting in higher acceptance rates than RRS. It even provably outperforms OTM under certain situations. The algorithm is widely applicable and can seamlessly integrate into various MDSD algorithms, enhancing their efficiency and overall decoding speed. 

We empirically test SpecHub by implementing it to various MDSD frameworks \cite{li2024eagle,chen2024sequoia,miao2023specinfer}. We observe a $1-5\%$ increase in the second draft acceptance rate, which yields a consistent $0.02-0.16$ improvement in batch efficiency over current methods. More impressively, SpecHub uses a tree with only half the nodes of other methods to reach the same level of batch efficiency. 
In our ablation study, SpecHub brings consistent acceleration to LLM decoding under different temperatures. 
Our toy experiments further show that SpecHub sometimes outperforms OTM in high-entropy regions.

\section{Background and Related Work}
\label{sec:background}

Here, we review the sampling-verification schema of speculative decoding. We discuss the theory behind rejection sampling and explain why naively extending it to Multi-Draft Speculative Decoding (MDSD) becomes inefficient. 

\paragraph{Speculative Sampling}
Language model decoding is intrinsically serial. Let $\mathcal{V}$ denote the vocabulary, a discrete set of tokens that the language model may generate. Let $x^{1:t} = (x^1, \ldots, x^{t})\in \mathcal{V}^{\otimes t} $ denote a sequence of tokens. Then, the target language model produces a conditional probability $p(\cdot|x^{1:t})$, from which we sample the next token $x^{t+1} \sim p(\cdot|x^{1:t})$. However, this process is slow for its serial execution. 

Speculative decoding \cite{chen2023accelerating, leviathan2023fast} addresses the issue by parallelizing the decoding process with a draft and verify phase. It first uses a smaller draft model $q(\cdot|x^{1:t})$ to generate a draft $(x^{t+1}, \ldots, x^{t+d})$ sequentially. The depth of the draft, $d$, is usually around $5$. This draft allows us to compute the target distributions $p(x^{t+\tau}|x^{1:t+\tau-1})$ in parallel for $\tau \leq d$. Then, we iteratively accept each draft token using rejection sampling with acceptance probability $\min\left(1,\frac{p(x^{t+\tau}|x^{1:t+\tau-1})}{q(x^{t+\tau}|x^{1:t+\tau-1})}\right)$. In this single draft setting, speculative decoding equates to sampling directly from the target distribution. After rejection, we sample from the residual distribution $\text{norm} (\max(0,p(\cdot|x^{1:t+\tau-1}) - q(\cdot|x^{1:t+\tau-1}))) $. \ty{It is better to explain the motivation by providing a visualization of the residual.}

With only a single draft, the expected number of tokens generated at each iteration is upper-bounded. Assume the average acceptance rate for each token is $\alpha$, the maximum acceleration is $1/(1-\alpha)$ \citep{chen2024sequoia}. Multi-Draft Speculative Decoding solves this issue \citep{miao2023specinfer,sun2024spectr}. Instead of verifying one sequence per time, MDSD generates a tree of tokens and calculates their target probability in parallel. Thus, when the first draft gets rejected, the other drafts can be picked up, and their offspring get verified in the current step. By doing so, we trade more parallel inference for more tokens generated in each step. 

\begin{figure}[t]
    \centering
    \includegraphics[width=0.45\textwidth]{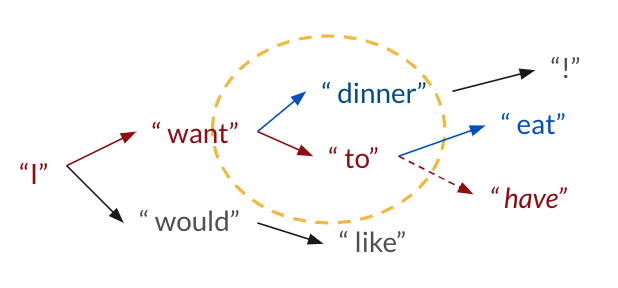}
    \caption{An example of a token tree of depth $d=4$ for MDSD. The tree is generated sequentially with the draft model and evaluated concurrently with the target model. Each path in the tree corresponds to a potential sequence of tokens, with \textcolor{darkbrown}{accepted tokens} and \textcolor{dimeblue}{rejected tokens} highlighted. The black arrows indicate tokens that were not visited. The dashed line represents a sample drawn from the residual distribution after all drafts are rejected. Our paper focuses on the evaluation of one step, how we choose to sample the $k=2$ tokens " dinner" and " to" from the draft distribution $q(\cdot|\text{"I want"})$ and decide which of them to get accepted based on the target probabilities $p(\text{" dinner"}|\text{"I want"})$ and $p(\text{" to"}|\text{"I want"})$.}
    \label{fig:token_tree}
\end{figure}

In the rest of the paper, we ignore any temporal relationship and only focus on a single temporal step in the decoding process. In particular, given $q(\cdot|x^{1:t-1})$ and $p(\cdot|x^{1:t-1})$, we discuss the sampling and verification algorithm for generating the offspring drafts and accepting one. 
We simplify the notation and use \( p = p(\cdot | x^{1:t-1}) \in \Delta^{|\mathcal{V}|-1}  \) \ty{$[0,1]^{|\mathcal{V}|}$ is not accurate, it should be a probability simplex.} to denote the target model's probability distribution and \( q = q(\cdot | x^{1:t-1}) \in \Delta^{|\mathcal{V}|-1} \) to denote the draft model's distribution. Here, $\Delta^{|\mathcal{V}|-1} = \left\{p \in \mathbb{R}^{|\mathcal{V}|} \ \bigg| \ \sum_{x \in \mathcal{V}} p(x) = 1, \ p(x) \geq 0 \ \forall x \in \mathcal{V} \right\}$ is the probability simplex of dimension $|\mathcal{V}|$. We also notate the probability simplex of joint distributions over a group of $k$ drafts at a single temporal step as $x_{1:k} = (x_1, \ldots, x_k)$ as: 
\begin{align*}
    \Delta^{|\mathcal{V}|^k - 1} = &\{ P \in \mathbb{R}^{|\mathcal{V}|^k} \ \bigg| \ \sum_{X \in \mathcal{V}^{\otimes k}}  P(x_{1:k}) = 1, \\ 
    &P(x_{1:k}) \geq 0 \ \forall x_{1:k} \in \mathcal{V}^{\otimes k} \}
\end{align*}

\paragraph{Rejection Sampling in Speculative Decoding} 

\ty{This part repeats most of the content in the Speculative Sampling part. Why?} We here provide a geometric intuition behind rejection sampling. 
Given a target distribution $p$ and a sample token from the draft distribution $x\sim q$, we seek to accept $x$ as much as possible while ensuring the outputted token from the process follows $p$. 
We can visualize the process as sampling a point under the probability mass function (PMF) of $p$. The draft sample lies under the PMF of $q$. If the token $x$ is undersampled ($q(x) < p(x)$), we always accept it.        
% it is always in the overlap of $p$ and $q$. 
If it is oversampled ($q(x) > p(x)$), the data point may or may not fall under $p$,
in which case we accept it with probability $p(x)/q(x)$, the height ratio between the two curves at this token. \ty{as shown in Figure ?} 
Such methods fully utilize the overlap between the two distributions and give the highest theoretical acceptance rate. See Figure~\ref{fig:rejection sampling}. \ty{how to get this conclusion?}

\begin{figure}[t]
    \centering
    \includegraphics[width=0.45\textwidth]{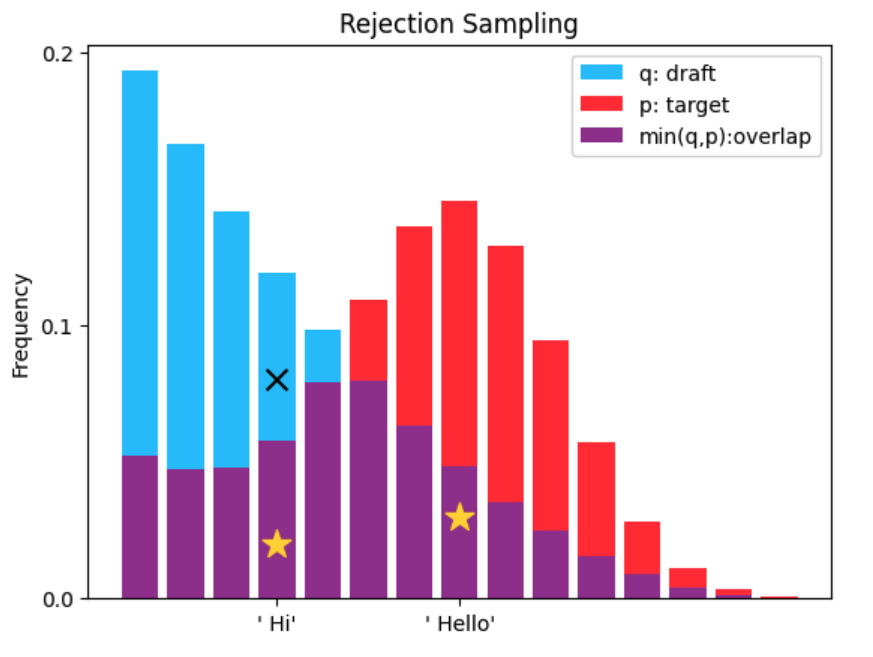}
    \caption{An illustration of rejection sampling. Sampling from the \textcolor[HTML]{33C1FF}{draft distribution} gives a point under the blue distribution $q$. If the sample is also under the \textcolor[HTML]{992C99}{overlap with the target distributions $p$}, we accept it. If not, we reject the token and sample from the residual distribution, the remaining unexplored area $\max(p-q, 0)$ normalized. The misalignment of the \textcolor[HTML]{FF2733}{residual distribution} and \textcolor[HTML]{33C1FF}{draft distribution} makes Recursive Rejection Sampling (RRS) inefficient in proceeding runs. }
    \label{fig:rejection sampling}
\end{figure}

The residual distribution $\text{norm} (\max(0,p - q))$ captures the remaining probability mass that was not covered by $q$. Sampling from this residual distribution ensures that any rejections are accounted for by exploring the regions where $p$ exceeds $q$. This approach aligns the accepted samples closely with $p$, effectively achieving maximal coupling and ensuring the samples represent the target distribution $p$.

\paragraph{Recursive Rejection Sampling}

To facilitate MDSD, previous methods use Recursive Rejection Sampling, which naively applies rejection sampling on the residual distributions. 
First, Recursive Rejection Sampling (RRS) samples $k$ candidates independently from the draft distribution. Then, it accepts each candidate with rejection sampling. If the token is rejected, the target distribution is updated to the residual distribution $\text{norm} (\max(p-q, 0))$.
While the acceptance of the first candidate is high, subsequent candidates suffer from the potential mismatch between the residual distributions and draft distribution \( q \). Essentially, our residual distribution deducts draft distribution, so we expect it to diverge from the draft distribution $q$ we used to generate our samples, leading to small overlapping areas and inefficiencies. 

\paragraph{Recursive Rejection Sampling without Replacement}

\begin{algorithm}[h]
\small
\caption{Token-level RRS}
\label{alg:rrs-wo}
\begin{algorithmic}[1]
\STATE {\bf Input:} Target model distribution $p$, draft model distribution $q$, number of candidates $k$
\STATE {\bf Output:} A token $x$ selected using RRS without replacement.
\STATE Generate $k$ samples $x_1, \ldots, x_k$ independently or \textcolor{red}{without replacement} from $q$
\FOR{$i= 1 \rightarrow k$}
\STATE sample $r_i \sim \mathrm{Uniform}(0,1)$
\IF {$r_i < \frac{p(x_i)}{q(x_i)}$}
\STATE {\bf Return $x_i$}
\ELSE
\STATE{$p \leftarrow \text{norm}(\max(p-q, 0))$}
\IF{\textcolor{red}{without replacement}}
\STATE{\textcolor{red}{$q(x_i) \leftarrow 0$}}
\STATE{\textcolor{red}{$q \leftarrow \text{norm}(q)$}}
\ENDIF
\ENDIF
\ENDFOR
\STATE{\bf Return $x \sim p$}
\end{algorithmic}
\end{algorithm}

In low-temperature settings, RRS may repeatedly sample the same token and fail to diversify the tree. Furthermore, a rejected token will continuously get rejected since the corresponding entry of the residual probability is $0$. Following this intuition, several works\cite{chen2024sequoia,jeon2024recursive, li2024eagle, yang2024multi} proposed Recursive Rejection Sampling without Replacement (RRSw). Instead of independently sampling, it samples tokens without replacement. It also modifies the draft distribution after each rejection to maintain a correct marginal distribution. The differences are highlighted in Algorithm~\ref{alg:rrs-wo} in red. While the method speeds up the decoding process by avoiding repetition, it still falls short of a theoretically optimal verification method as the misalignment between residual distribution and the draft distribution remains.

\section{Mathematical Formulation of Multi-Draft Speculative Decoding}
\label{sec:math}

In this section, we lay out the mathematical formulation of the sampling and verification paradigm of MDSD. We start by reviewing the Optimal Transport with Membership Cost framework by \citet{sun2024spectr} in Section~\ref{subsec:OTM}. We show that it can simplified and propose an equivalent LP formulation that greatly reduces computation complexity in Section~\ref{subsec:LP}. Lastly, we point out that changing the design of sampling can make the LP feasible for real-world calculation in Section~\ref{subsec:sampling} while preserving the acceleration. We also discuss some considerations for a real-world algorithm. 

\subsection{Optimal Transport with Membership Cost}
\label{subsec:OTM}

We show how we can find the optimal sampling and verification algorithm of MDSD that maximizes the acceptance rate as solving an Optimal Transport problem with Membership Cost~\citep{sun2024spectr}. 
Let the target distribution be $p$ and the joint draft distribution $Q = q^{\otimes k}\in \Delta^{|\mathcal{V}|^k -1}$ be the Cartesian product of the draft distributions that gives the probability of sampling any particular series of draft tokens $x_{1:k}$, so $Q(x_{1:k}) = \prod_{i=1}^k q(x_i)$. Let $y$ denote the accepted token.
We define the coupling between $p$ and $Q$ or equivalently a transport plan from $Q$ to $p$ be a joint distribution $\pi(x_{1:k}, y)\in\Delta^{|\mathcal{V}|^{k+1} -1} $ whose marginal distributions satisfies $\sum_{y\in\mathcal{V}} \pi(x_{1:k}, y) = Q(x_{1:k})$ and $\sum_{x_{1:k}\in\mathcal{V}^k}\pi(x_{1:k}, y) = p(y)$. We use the terms coupling and transport plan interchangeably. 
The Membership Cost is \(
c(x_{1:k}, y)  = \prod_{i=1}^k \mathds{1}_{y \neq x_i}
\), an indicator function of whether the accepted token $y$ equals any of the draft tokens $x_i$. The transport cost then calculates the expected rejection rate:
\[
C(\pi) = \mathbb{E}_{x_{1:k}, y \sim \pi} \left[\prod_{i=1}^k \mathds{1}_{y \neq x_i}\right].
\]
It is well-known that Optimal Transport on discrete space can be solved as a linear programming problem as
\begin{equation}
\label{eq:OT_LP}
\min_{\pi \in \Pi(p,q)} \sum_{x_{1:k}} \sum_{y \in \mathcal{V}} \pi(x_{1:k}, y) \prod_{i=1}^k \mathds{1}_{y \neq x_i}
\end{equation}
where $\Pi(p,q)$ is the set of all valid couplings between $p$ and $q^{\otimes k}$. However, such a program contains $O(|\mathcal{V}|^{k+1})$ variables, so even the fastest linear programming algorithm struggles to calculate in real-time. 

\begin{figure}[t]
    \centering
    % Subfigure 3
    \begin{subfigure}{0.45\textwidth}
        \centering
        \includegraphics[width=\textwidth]{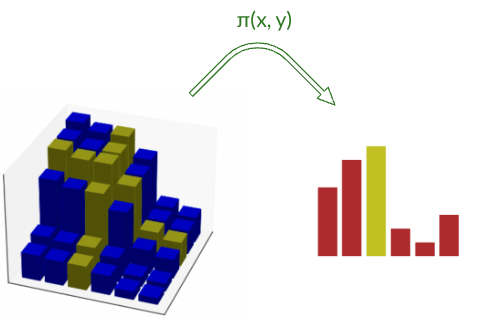}
        \caption{OTM}
        \label{subfig:OTM_pi}
    \end{subfigure}
    \hfill % Space between the subfigures
    \begin{subfigure}{0.45\textwidth}
        \centering
        \includegraphics[width=\textwidth]{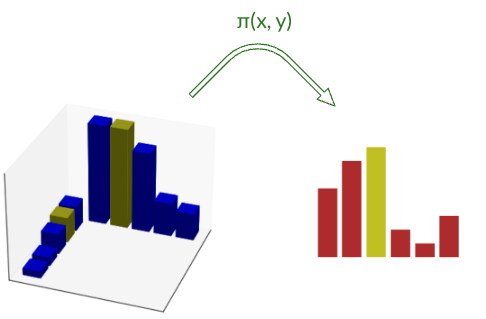}
        \caption{SpecHub}
        \label{subfig:SpecHub_pi}
    \end{subfigure}

    \caption{An illustration comparing the Optimal Transport with Membership Cost (OTM) framework and SpecHub. In both (a) and (b), the left side shows a two-draft joint sampling distribution, while the right side depicts the target distribution. The yellow bars highlight the token of interest in the target. In (a), OTM requires solving for the transport map $\pi$ of a dense sampling distribution like $Q = q^{\otimes 2}$, which is computationally expensive. In (b), SpecHub simplifies this process by sparsifying the joint distribution, significantly reducing the complexity of solving for $\pi$.}

\end{figure}

\subsection{A Simplified Linear Programming Formulation}
\label{subsec:LP}
While the Optimal Transport formulation provides a theoretical framework for understanding Multi-Draft Speculative Decoding, its computational complexity renders it impractical for real-time applications. To address this, we introduce a simplified Linear Programming (LP) formulation that significantly reduces the number of variables while preserving the essence of the problem.

The key insight behind this simplification is that the acceptance rate is primarily determined by how the sampled draft tokens are handled. Once a token is rejected, the subsequent actions, which involve recalculating the residual distribution and resampling, can be performed efficiently without explicitly considering the full coupling. 

Instead of representing the entire coupling $\pi$, which has $O(|\mathcal{V}|^{k+1})$ variables, our simplified LP formulation focuses on $\pi(x_{1:k}, y=x_i)$, $i = 1,\ldots,k$, a smaller subset of transport plan which denotes the probability of sampling the series of drafts  and accepting the $i$-th token $x_i$. This effectively reduces the number of variables to $O(|\mathcal{V}|^k)$, making the problem more tractable.
The remaining probabilities in the coupling, which correspond to cases where the target token does not match any of the draft tokens, are implicitly handled by the residual distribution.

% The simplified LP formulation is then:
% \begin{align*}&\text{minimize}_{\pi}  \  1- \sum_{x_{1:k}\in\mathcal{V}^k}\sum_{i=1}^k  \pi(x_{1:k}, x_i)\\
%     &\text{subject to} \quad \\
%     &\pi(x_{1:k}, x_i) \geq 0 
%     &&\forall x_{1:k}\in\mathcal{V}^k,i\\
%     & \sum_{i=1}^k \pi(x_{1:k}, x_i) \leq Q(x_{1:k}) &&\forall  x_{1:k}\in\mathcal{V}^k \\
%     & \sum_{i=1}^k \sum_{ x_{1:k}\in\mathcal{V}^k,x_i=y} \pi(x_{1:k}, y) \leq p(y)&&\forall y \in \mathcal{V}
% \end{align*}

The simplified LP formulation is then:
\begin{align*}&\text{minimize}_{\pi}  \  1- \sum_{x_{1:k}\in\mathcal{V}^k}\sum_{i=1}^k  \pi(x_{1:k}, x_i)\\
    &\text{subject to}  \\
    &\pi(x_{1:k}, x_i) \geq 0 
    ,\quad\quad\quad\quad\quad\quad\ \forall x_{1:k}\in\mathcal{V}^k,i\\
    & \sum_{i=1}^k \pi(x_{1:k}, x_i) \leq Q(x_{1:k}) ,\quad\quad\ \ \forall  x_{1:k}\in\mathcal{V}^k \\
    & \sum_{i=1}^k \sum_{ x_{1:k}\in\mathcal{V}^k,x_i=y} \pi(x_{1:k}, y) \leq p(y),\quad\forall y \in \mathcal{V}
\end{align*}

% \begin{align*}
% &\text{minimize}_{\pi}  \quad  1 - \sum_{x_{1:k}\in\mathcal{V}^k} \sum_{i=1}^k  \pi(x_{1:k}, x_i) \\
% &\text{subject to} \\
% &  \pi(x_{1:k}, x_i) \geq 0 
% & \forall x_{1:k}\in\mathcal{V}^k, i \\
% & \sum_{i=1}^k \pi(x_{1:k}, x_i) \leq Q(x_{1:k}) 
% & \forall x_{1:k}\in\mathcal{V}^k \\
% & \sum_{i=1}^k \sum_{x_{1:k}\in\mathcal{V}^k, x_i=y} \pi(x_{1:k}, y) \leq p(y) \\
% & \forall y \in \mathcal{V}
% \end{align*}

Given a solution to this simplified LP formulation, we can reconstruct the complete transport plan $\pi(x_{1:k}, y)$. For any series of drafts $x_{1:k}$ and target token $y$, if $y$ does not equal one of the draft tokens in $x_{1:k}$, the entry is calculated as:
\begin{align*}
&\pi(x_{1:k}, y)  \quad\text{\# where }y\neq x_i\ \forall i=1,\ldots,k\\
=\ &\frac{p(y) - \sum_{i=1}^k \sum_{ x_{1:k}\in\mathcal{V}^k,x_i=y} \pi(x_{1:k}, y)}{\sum_{y\in \mathcal{V}}p(y) - \sum_{i=1}^k \sum_{ x_{1:k}\in\mathcal{V}^k,x_i=y} \pi(x_{1:k}, y)}\\
&\cdot (Q(x_{1:k}) - \sum_{i=1}^k \pi(x_{1:k}, x_i))
\end{align*}

The first term is the unallocated target probability mass or the residual probability of $y$ normalized
%$\frac{p(y) - \sum_{i=1}^k \sum_{ x_{1:k}\in\mathcal{V}^k,x_i=y} \pi(x_{1:k}, y)}{\sum_{x\in \mathcal{V}} p(x) - \sum_{i=1}^k \sum_{ x_{1:k}\in\mathcal{V}^k,x_i=x} \pi(x_{1:k}, x)}$
. The second term is the remaining probability mass of the series of drafts $x_{1:k}$ after allocating probabilities to cases where the target token matches a draft token.
%, which is $Q(x_{1:k}) - \sum_{i=1}^k \pi(x_{1:k}, x_i)$.
This reconstruction process ensures that the validity of the coupling.
This simplified LP formulation, while ignoring the explicit representation of the full coupling, retains the essential information needed to optimize the acceptance rate. It provides a practical and computationally feasible approach to solving the MDSD problem.

\begin{theorem}[Equivalence of LP to OTM]
\label{thm:LP_equivalence}
For a given joint draft distribution $Q$ and target distribution $p$, the optimal solution of the simplified LP formulation achieves the same transport cost as the maximal coupling in the Optimal Transport with Membership Cost (OTM) problem, i.e., $1- \sum_{x_{1:k}\in\mathcal{V}^k}\sum_{i=1}^k  \pi(x_{1:k}, x_i) = C(\pi^*)$, where $\pi^*$ is the optimal coupling for the OTM problem as defined in Equation \ref{eq:OT_LP}.
\end{theorem}

\begin{proof}
    See Appendix~\ref{appendix:LP_correctness}.
\end{proof}
\paragraph{Examining Recursive Rejection Sampling (RRS)} How does an optimal solution to the Linear Programming (LP) formulation differ from RRS? Consider the simple case of $k=2$. When a series of drafts $x_1, x_2$ is sampled according to $Q(x_{1:2})$, we must decide whether to accept $x_1$ or $x_2$ based on the target distribution $p$. If $x_1$ is significantly over-sampled, meaning $p(x_1) < q(x_1)$. RRS makes this decision independently for each draft token, while the OTM solution considers the entire series. Specifically, the OTM solution will tend to allocate less probability mass to accepting $x_1$ if $x_2$ is undersampled ($p(x_2) > q(x_2)$) and more probability mass if $x_2$ is also oversampled. This flexible adaptation ensures a more targeted distribution in subsequent drafts, leading to more efficient sampling and verification. 

\begin{figure}[t]
    \centering
    % Subfigure 1
    \iffalse
    \begin{subfigure}{0.4\textwidth}
        \centering
        \includegraphics[width=\textwidth]{figures/LP_with.png}
        \caption{Subfigure 1: Placeholder}
        \label{subfig:LP_sampling}
    \end{subfigure}
    \hfill % Space between the subfigures
    % Subfigure 2
    \begin{subfigure}[b]{0.45\textwidth}
        \centering
        \includegraphics[width=\textwidth]{figures/LP_RRS_with.png}
        \caption{Subfigure 2: Placeholder}
        \label{subfig:LP_RRS_w}
    \end{subfigure}
    \fi
    % Subfigure 3
    \begin{subfigure}{0.4\textwidth}
        \centering
        \includegraphics[width=\textwidth]{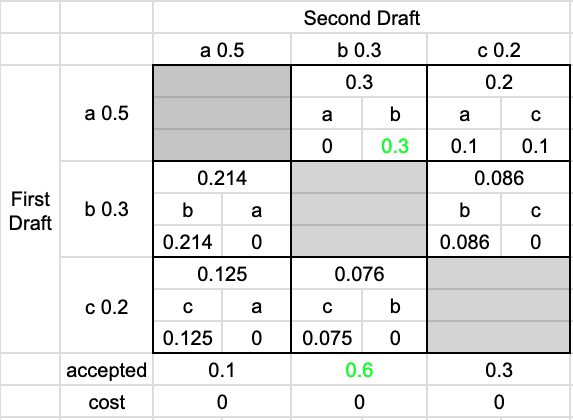}
        \caption{Optimal solution to LP}
        \label{subfig:LP_recursive_sampling}
    \end{subfigure}
    \hfill % Space between the subfigures
    \begin{subfigure}{0.4\textwidth}
        \centering
        \includegraphics[width=\textwidth]{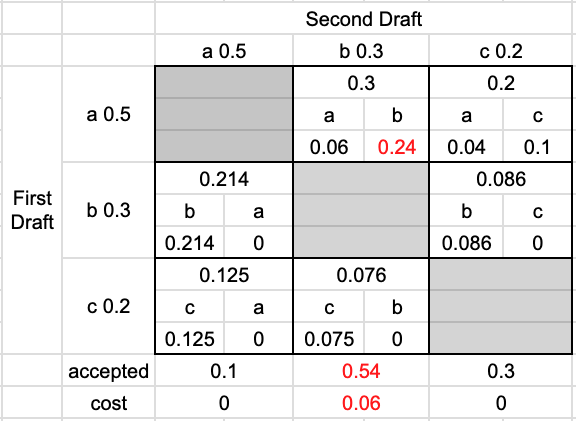}
        \caption{RRSw solution to LP}
        \label{subfig:LP_RRS_sampling_wo}
    \end{subfigure}

    \caption{A comparison of an optimal solution to an RRSw solution under the LP formulation. Here, the draft distribution $q = [0.5,0.3,0.2]$ and the target distribution $p = [0.1, 0.6, 0.3]$. Each number on the top of the cell is $Q(x_1, x_2)$, and the numbers at the bottom of the cell show $\pi(x_1,x_2, x_1)$ and $\pi(x_1,x_2, x_2)$, i.e. how much of those draft probabilities are transferred to the target probability. RRSw has a transport cost of $0.06$ for not generating enough token 'b'.}

\end{figure}

\paragraph{Unbalanced Tree and Asymmetric Verification} When considering a single temporal step in the sampling and verification process, the order in which a pair of samples $x_{1:k}$ is selected appears inconsequential, as the branches are executed concurrently. However, as suggested by Sequoia~\cite{chen2024sequoia}, the most efficient tree structure is often unbalanced. If the acceptance rate of the early draft is higher than that of the second, designing a tree that extends deeper along the first few branches while keeping other branches shallower can enhance efficiency.
Optimal algorithms may decrease the first few drafts' acceptance rate slightly to achieve a higher overall acceptance rate, which we need to carefully balance to leveraging the advantages of unbalanced tree structures and significantly improving decoding speed and performance.

\subsection{Design of Sampling}
\label{subsec:sampling}

While the simplified LP formulation significantly reduces the computational burden compared to the OTM, it remains computationally expensive for large vocabularies. Directly solving the LP problem is impractical, and previous research has predominantly focused on developing heuristics to approximate the optimal solution. These heuristics, such as Recursive Rejection Sampling (RRS) or SpecTr\cite{sun2024spectr}, operate under a fixed joint draft distribution, typically assuming independent sampling with ($Q = q^{\otimes k}$) or without replacement ($Q(x_{1:k}) = \frac{\prod_{i=1}^k q(x_i)}{\prod_{i=1}^{k-1} (1-\sum_{j=1}^{i}q(x_j))}$).

However, a crucial and often overlooked aspect is the ability to \textbf{modify the joint draft distribution $Q$}, which unlocks a new dimension for optimization that has not been fully explored. 
The key to designing a practical and efficient sampling strategy is recognizing that $Q$ does not need to be a dense distribution over all possible drafts. Instead, we can strategically construct a sparse $Q$ that simplifies the LP formulation while capturing the essential features of the target distribution. This sparsity reduces the number of variables and constraints in the LP, making it significantly easier to solve or approximate. 

Ideally, the design of $Q$ should satisfy two key criteria: 1) \textbf{Sparsity}; $Q$ should be sparse, concentrating on a small subset of highly probable draft series to reduce computational complexity; and 2) \textbf{Efficiency}; $Q$ should effectively capture the essential features of target distribution $p$, ensuring that the sampled drafts are likely to contain the target token.
By carefully designing $Q$, we can balance computational efficiency and acceptance rate, paving the way for practical and high-performance MDSD algorithms.

\section{SpecHub}

Building on the aforementioned insights, we introduce SpecHub, a faster sampling-and-verifying paradigm with only linear computational overhead. It effectively captures the transport features of OTM solutions to enhance the acceptance rate and can be applied to various multi-draft speculative sampling algorithms. 
Since using more than two drafts offers little gains in efficiency, SpecHub uses two drafts (i.e., $k=2$) to reduce complexity. We thoroughly discuss expanding the algorithm to more drafts in Appendix~\ref{appendix:more_drafts}.

First, we identify the token with the highest draft probability, denoted as \(a\), and sample it alongside other tokens. We only populate the first column and the first row in the joint draft distribution $Q$. In particular, we define the joint draft distribution \(Q(x_1, x_2)\) as follows:
\[
Q(x_1, x_2) = 
\begin{cases} 
q(x_1) & \text{if } x_2 = a, \\
\frac{q(a)q(x_2)}{1-q(a)} & \text{if } x_1 = a, \\
0 & \text{otherwise}.
\end{cases}
\]
This specific design of $Q$ makes the solution to the simplified LP formulation straightforward. $\forall x \in \mathcal{V}, x\neq a$, we have 
\iffalse
\begin{align*}
    \pi(x, a, x) &= \min(p(x), q(x)) \\
    \pi(a, x, x) &= \min(p(x) - \pi(x,a,x), Q(a,x)) \\
    \pi(a, x, a) &= \min(p(a), \sum_{x\in \mathcal{V}} (Q(a,x) - \pi(a,x,x)) \\
    &\quad \cdot \frac{Q(a,x) - \pi(a,x,x)}{\sum_{x\in \mathcal{V}} (Q(a,x) - \pi(a,x,x))}\\
    \pi(x, a, a) &= \min(p(a) - \sum_x \pi(a, x, a),\\
    &\quad\quad\quad \sum_{x\in\mathcal{V}} q(x) - \pi(x, a, x))\\
    &\cdot \frac{ q(x) - \pi(x, a, x)}{\sum_{x\in\mathcal{V}} q(x) - \pi(x, a, x)}
\end{align*}
\fi
\begin{align*}
    \pi(x, a, x) &= \min(p(x), q(x)) \\
    \pi(a, x, x) &= \min(p(x) - \pi(x,a,x), Q(a,x)) 
\end{align*}
After transporting draft probabilities to target probabilities of non-top tokens, the remaining draft accepts the top token $a$ evenly out of $p(a)$
The remaining entries in $\pi$ can be reconstructed as described in the previous section. This solution effectively allocates as much probability mass as possible to the non-hub draft tokens while ensuring that the hub token $a$ is never undersampled. This strategy maximizes the utilization of the draft distribution and leads to a higher acceptance rate compared to traditional methods like RRS. The pseudocode implementation is in Appendix~\ref{appendix:pseudocode}.  

\begin{figure}[t]
    \centering
    \includegraphics[width=0.4\textwidth]{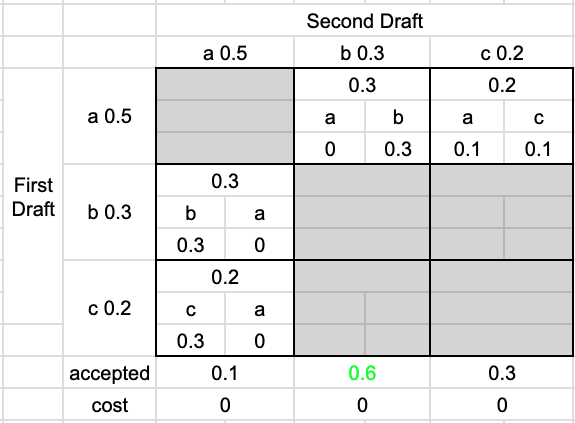}
    \caption{ SpecHub under the LP formulation. Here, the draft distribution $q = [0.5,0.3,0.2]$ and the target distribution $p = [0.1, 0.6, 0.3]$. SpecHub focuses on the top token "a", sampling pairs $(x, a)$ and $(a, x)$ with probabilities $q(x)$ and $\frac{q(a)q(x)}{1-q(a)}$, respectively. This method ensures efficient allocation of acceptance probabilities.}

    \label{fig:LP_SpecHub}
\end{figure}

\paragraph{Analysis} SpecHub offers several theoretical advantages. First, since all drafts contain the top token $a$, it is accepted with a probability of $p(a)$ and is never undersampled (see Corollary~\ref{lemma:top_token}). Additionally, 
%let $q$ and $p$ denote the draft and target distributions, respectively, and 
let $\alpha$ be the first draft acceptance rate of rejection sampling, defined as $\alpha = \sum_x \max(p(x), q(x))$. SpecHub achieves a higher acceptance rate than Recursive Rejection Sampling (RRS) if the top token $a$ satisfies the condition $\frac{q(a)}{1-q(a)} > 1 - \alpha$. We even guarantee acceleration over OTM sampled with replacement $Q=q^{\otimes 2}$ if $\frac{1}{1-q(a)} > 2$ or $q(a) > \frac{1}{2}$. Detailed proofs of these results are in Appendix~\ref{appendix:acceptance_rate}.

While SpecHub might theoretically decrease the first draft acceptance rate for the top token $a$ in rare cases, our empirical results, detailed in Appendix~\ref{appendix:first_draft}, show that this effect is negligible. 

\section{Experiments}

In this section, we empirically show that SpecHub improves batch efficiency in speculative multi-draft decoding. We first show that SpecHub gives a significantly higher acceptance rate for its better coupling properties in the second draft acceptance rate. We then illustrate how the improvement transfers to higher batch efficiency. 

\subsection{Experiment Setup}

Our experimental setup is based on the Llama and Vicuna models. To mimic the setup of \citet{chen2024sequoia}, we utilize the JackFram/Llama-68m and JackFram/Llama-160m (JF68m, JF160m)~\cite{miao2023specinfer} models as our draft models and the Llama2-7B~\cite{touvron2023llama} models as our target models.  We evaluate our results on the OpenWebText~\cite{Gokaslan2019OpenWeb} and CNN DailyMail~\cite{see-etal-2017-get} datasets. For each run, we use 200 examples to measure the acceptance rate vector and sample another 200 examples for evaluation. The prompt length and generation length are both set to 128 tokens. We evaluate our system on a single RTX A5000 GPU. 

We also implement our algorithm on EAGLE \citep{li2024eagle}. In short, EAGLE trains an autoregressive decoding head that takes both the embedding in the last layer of the target model and the draft tokens to predict a draft. We test its performance on Vicuna-7b \cite{zheng2024judging}, a fine-tuned LLaMA chatbot using ChatGPT \citep{openai2024gpt4} to generate responses. We use the MT-Bench dataset and temperatures $T=0.6,\ 1.0$ with binary trees and binary Sequoia trees.

\subsection{Main Experiments}

\begin{table}
\centering

\begin{tabular}{cccc}
\toprule
T & RRS & RRSw & \textbf{SpecHub} \\
\midrule
0.3 & 0.0426 & 0.1114 & \textbf{0.1184} \\
0.6 & 0.0740  & 0.1089 & \textbf{0.1379} \\
1.0 & 0.1021 & 0.1140  & \textbf{0.1660}  \\
\bottomrule
\end{tabular}

\caption{Second Draft Acceptance Rate for JF68m Model} %
\label{tab:68m-model}
\vspace{2ex}

\centering

\begin{tabular}{cccc}
\toprule
T & RRS & RRSw & \textbf{SpecHub} \\
\midrule
0.3 & 0.0399 & 0.1129 & \textbf{0.1221} \\
0.6 & 0.0730  & 0.1212 & \textbf{0.1351} \\
1.0 & 0.0910  & 0.1176 & \textbf{0.1660}  \\
\bottomrule
\end{tabular}

\caption{Second Draft Acceptance Rate for the JF160m Model}
\label{tab:160m-model}
\end{table}

\paragraph{Second Draft Acceptance Rate}

We evaluate SpecHub at different temperatures $T = 0.3, 0.6, 1.0$ using JF68m and JF160m as draft models. We observe that SpecHub consistently outperforms RRS and RRSw. In particular, at higher temperatures, SpecHub achieves up to $5\%$ improvements in the second draft acceptance rate from $0.114-0.118$ to $0.166$. At a lower temperature, the improvement over RRSw becomes smaller since the whole process assimilates greedy decoding. In fact, SpecHub is equivalent to RRS without replacement at zero temperature since both algorithms become top-2 greedy decoding. Results are shown in Table~\ref{tab:68m-model} and \ref{tab:160m-model}. 

\paragraph{Batch Efficiency}

\begin{figure}[t]
    \centering
    \includegraphics[width=\linewidth]{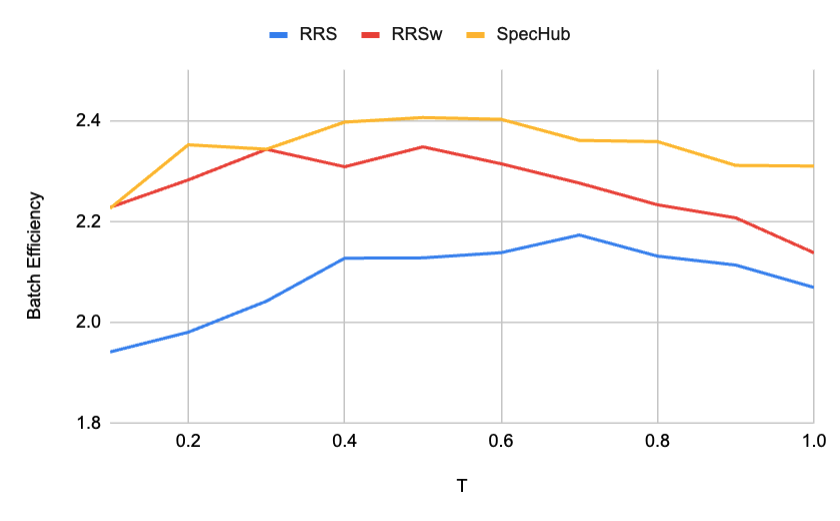}
    \caption{Batch efficiency at different temperatures.}
    \label{fig:batch_efficiency_temperature}
\end{figure}

We examine how the increased second-draft acceptance rate translates to better batch efficiency in different tree configurations. We empirically test SpecHub and RRS without replacement on binary trees of depth $d$ with $2^d-1$ nodes and report the batch efficiency in \ref{fig:batch_efficiency_comparison}. We see that with JF68M as the draft model, SpecHub consistently outperforms RRS and RSSw by $0.02-0.10$ and $0.04-0.20$ in batch efficiency at temperatures $T=0.6,\ 1.0$. Meanwhile, using the EAGLE decoding head as the draft model, SpecHub generates up to $3.53$ and $3.33$ tokens per iteration in the binary tree setting at $T=0.6,\ 1.0$, an additional $0.08$ tokens than RRS without replacement. We also tested the batch efficiency on optimal binary Sequoia trees\cite{chen2024sequoia}. The full experiment results are in Appendix~\ref{appendix:experiments}. 

\subsection{Ablations}
\label{subsec:ablations}
\paragraph{Robustness to Different Temperatures}

We analyze the performance of SpecHub across different temperatures (T) and compare it with Recursive Rejection Sampling (RRS) and RRS without replacement (RRSw). We use a binary token tree of depth $d=5$ with JF68m as the draft model for Llama-2-7b. As shown in Figure~\ref{fig:batch_efficiency_temperature}, SpecHub consistently outperforms both RRS and RRSw regarding batch efficiency across all temperature settings. At lower temperatures $(T < 0.4)$, SpecHub assimilates RRSw in performance. At medium $(0.4 \leq T \leq 0.6)$ and higher temperatures $(T > 0.6)$, SpecHub maintains superior performance, demonstrating its robustness and adaptability across varying entropy levels. 

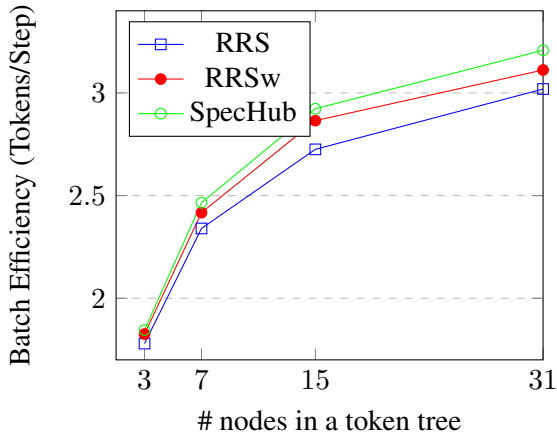
\begin{figure}[t]
    \centering
    \begin{tikzpicture}
        \begin{axis}[
            xlabel=\# nodes in a token tree,
            ylabel=Batch Efficiency (Tokens/Step),
            xmin=1, xmax=31,
            ymin=1.7, ymax=3.4, % Adjust y-axis limits as needed
            xtick={3,7,15,31},
            legend pos=north west,
            ymajorgrids=true,
            grid style=dashed,
            width=0.45\textwidth, % Adjust width as needed
            % height=0.6\textwidth, % Adjust height as needed
        ]
        \addplot[
            color=blue,
            mark=square,
        ]
            coordinates {
                (3, 1.7792) (7, 2.3399) (15, 2.7252) (31, 3.0190)
            };
            \addlegendentry{RRS}
        \addplot[
            color=red,
            mark=*,
        ]
            coordinates {
                (3, 1.8257) (7, 2.4169) (15, 2.8649) (31, 3.1116)
            };
            \addlegendentry{RRSw}
        \addplot[
            color=green,
            mark=o,
        ]
            coordinates {
                (3, 1.8454) (7, 2.4650) (15, 2.9230) (31, 3.2079)
            };
            \addlegendentry{SpecHub}
        \end{axis}
    \end{tikzpicture}
    \caption{Decoding efficiency of Vicuna-33B on the MT-Bench dataset at temperature $T=1.0$.}
    \label{fig:larger_model}
\end{figure}

\paragraph{Scaling to Larger Models}
We evaluate SpecHub on larger models, specifically Llama 2-13B-Chat and Vicuna-1.3-33B, to assess its scalability. For Llama, we use JackFram/Llama-160m as the draft model and test on the CNN DailyMail dataset. For Vicuna, we evaluate using the MT-Bench dataset with an EAGLE decoding head. Both experiments are conducted at $T=1.0$ with prompt and generation lengths set to 128 tokens. Results in Figure~\ref{fig:larger_model} show that SpecHub consistently outperforms RRS and RRSw, generating up to 0.29 more tokens per step on Llama and 0.19 tokens per step on Vicuna. SpecHub extends to larger models well, maintaining its efficiency gains as model size increases.

\section{Related Work}

\paragraph{Speculative Decoding}

Speculative decoding aims to execute multiple decoding steps in parallel. Early work \cite{stern2018blockwise} predicts future tokens to accelerate greedy decoding. 
Speculative Sampling \cite{chen2023accelerating,leviathan2023fast} extends to non-greedy decoding and uses rejection sampling to recover target distribution optimally. 
%Several subsequent works followed but replaced the future prediction with a much smaller draft model \cite{gante2023assisted} and used adaptive drafting depth \cite{kim2024biglittle}.
Recent works focus on reducing the running time of the draft model and increasing the acceptance rate. OSD \cite{liu2023online} and DistillSpec \cite{zhou2023distillspec} train draft models on text generated by the target model.  REST \cite{he2023rest} constructs drafts through retrieval. Lookahead Decoding \cite{fu2024break} breaks the sequential dependency with Jacobi Iterations. Self-Speculative Decoding \cite{zhang2023draft,elhoushi2024layer} avoids additional models and generates draft tokens by skipping intermediate layers. Several works, such as MEDUSA \cite{cai2024medusa} and EAGLE \cite{li2024eagle}, reuse the feature embedding of LLMs' last attention layer to predict multiple future tokens in a non-causal or autoregressive manner. 

\paragraph{Multi-Draft Speculative Decoding}

Recent research explores using tree attention to generate multiple drafts for speculative decoding \cite{miao2023specinfer, spector2023accelerating,li2024eagle}. \citet{sun2024spectr} formulate the acceptance of multiple drafts as a maximal coupling problem between the drafts and the target distributions and propose SpecTr with $1-\frac{1}{e}$ optimality guarantee. CS Drafting \cite{chen2023cascade} swaps in a lower-quality model to generate drafts for less relevant branches. MEDUSA \cite{cai2024medusa} establishes candidates according to the Cartesian product of the multi-head predictions.
Independently, \citet{jeon2024recursive} and \citet{yang2024multi} notice that a rejected token has zero probability in the residual distribution and use sampling-without-replacement in the draft generation. \citet{hu2024accelerated} accelerates MDSD Tree Monte Carlo methods, which treat language model generation as a tree-space problem. Sequoia \cite{chen2024sequoia} designed a dynamic programming algorithm to search for the optimal tree topology.

\section{Conclusion}

We presented SpecHub, a versatile and provably faster sampling-verification paradigm for Multi-Draft Speculative Decoding. Using an optimal transport map between a sparse draft and target distributions, SpecHub increases the acceptance rate of the second draft by $1-5 \%$, which leads to higher batch efficiency of LLM inference by up to 0.27 tokens per iteration. In addition to practical speedups, SpecHub also provides insight into the underlying mathematical structure in MDSD. We hope to promotes future research in this area. 

\section*{Acknowledgement} This work is done during an Internship at Samsung Research America. We thank Zhengmian Hu, Yixin Liu, Lichang Chen, and Qi Pang for their helpful discussion. 

\section*{Limitations}

Our algorithm, SpecHub, is currently designed to support only two drafts due to the computational complexities associated with using more drafts. This limitation may affect users who rely heavily on large-scale parallel computations, particularly when the number of nodes in the token tree exceeds $32$. However, such extensive parallelism is rarely utilized in practical applications, and most users will not encounter this limitation.

\section*{Ethical Statement}
This work focuses on accelerating LLM inferencing. There are no potential risks or negative effects that the authors are aware of. Additionally, we ensured that all datasets and benchmarks used in the article comply with their intended purposes and standards.

% Bibliography entries for the entire Anthology, followed by custom entries
%\bibliography{anthology,custom}
% Custom bibliography entries only
\bibliography{custom}

\appendix

\section{Correctness of the LP formulations}
\label{appendix:LP_correctness}

We prove Theorem~\ref{thm:LP_equivalence} to show that the simplified LP formulation is equivalent to the Optimal Transport with Membership Cost (OTM) problem.

\begin{proof}
We first show that we can construct a valid coupling from a valid solution to the simplified LP formulation. Given a solution represented by $\pi(x_{1:k}, x_i)$, we can derive a complete coupling $\pi(x_{1:k}, y)$, which represents the joint probability distribution of the $k$ draft tokens $x_{1:k}$ and the target token $y$.

The construction process involves allocating probabilities based on the LP solution. For each possible combination of draft tokens and target token $(x_{1:k},y)$, if $y$ matches any of the draft tokens, meaning $y = x_i$ for some $i$, then the corresponding entry in the transport plan is given by the solution to the LP:
\[\pi(x_{1:k},y) = \pi(x_{1:k}, x_i)\] 

% If the target token $y$ is different from all draft tokens, the probability is calculated as the product of two terms:
%     \begin{align*}
%     &\pi(x_{1:k},y) \\
%     &=\ \frac{p(y) - \sum_{i=1}^k \sum_{ x_{1:k}\in\mathcal{V}^k,x_i=y} \pi(x_{1:k}, y)}{\sum_{y\in \mathcal{V}}p(y) - \sum_{i=1}^k \sum_{ x_{1:k}\in\mathcal{V}^k,x_i=y} \pi(x_{1:k}, y)}\\
%     &\quad\cdot (Q(x_{1:k}) - \sum_{i=1}^k \pi(x_{1:k}, x_i))
%     \end{align*}
% The first term is the unallocated target probability mass or the residual probability of $y$ normalized. The second term is the remaining probability mass of the series of drafts $x_{1:k}$ after allocating probabilities to cases where the target token matches a draft token.

Let $\alpha(y) =\sum_{i=1}^k \sum_{ x_{1:k}\in\mathcal{V}^k,x_i=y} \pi(x_{1:k}, y)$ be the accepting probability of token $y$. If the target token $y$ is different from all draft tokens, the probability is calculated as the product of two terms:
    \begin{align*}
    &\pi(x_{1:k},y) \\
    &=\ \frac{p(y) - \alpha(y)}{\sum_{y\in \mathcal{V}}(p(y) - \alpha(y))}\\
    &\quad\cdot (Q(x_{1:k}) - \sum_{i=1}^k \pi(x_{1:k}, x_i))
    \end{align*}
The first term is the unallocated target probability mass or the residual probability of $y$ normalized. The second term is the remaining probability mass of the series of drafts $x_{1:k}$ after allocating probabilities to cases where the target token matches a draft token.

We now verify that the constructed $\pi$ is indeed a valid coupling. 
First, we need to show that the marginal distribution on the target token $y$ is indeed $p(y)$:
\begin{align*}
    &\sum_{x_{1:k}} \pi(x_{1:k}, y) \\
    =&\sum_{i=1}^k \sum_{x_{1:k},x_i=y} \pi(x_{1:k}, y)\\
    &+ (p(y) - \sum_{i=1}^k \sum_{ x_{1:k},x_i=y} \pi(x_{1:k}, y)) \\
    &= p(y). 
\end{align*}
Then, we verify that the marginal distribution on the series of drafts is the joint draft distribution:
\begin{align*}
    &\sum_{y} \pi(x_{1:k},y) \\
    = &\sum_{i=1}^k \pi(x_{1:k}, x_i) \\
    + &\sum_{y\neq x_i \forall i} (\frac{p(y) - \alpha(y)}{\sum_{y\in \mathcal{V}}(p(y) - \alpha(y))} \\
    &\cdot (Q(x_{1:k}) - \sum_{i=1}^k \pi(x_{1:k}, x_i)))\\
    = &Q(x_{1:k})
\end{align*}

Now, we show that an optimal solution to the simplified LP formulation is also optimal for the OTM problem. 

We prove this by contradiction. Assume there exists a coupling $\pi'$ that achieves a lower transport cost than the optimal solution to the simplified LP formulation. We can construct a solution $\pi''(x_{1:k}, x_i)$ to the LP from $\pi'$ by setting $\pi''(x_{1:k}, x_i) = \pi'(x_{1:k}, x_i)$. This $\pi''$ will have the same objective value as the transport cost of $\pi'$, contradicting the optimality of the LP solution. Therefore, an optimal solution to the simplified LP formulation is also an optimal solution to the OTM problem.
\end{proof}

\section{Properties of SpecHub}

\subsection{Pseudocode Implementation of SpecHub}
\label{appendix:pseudocode}

The transport plan of top token $a$ is: 

\begin{align*}
    \pi(x, a, x) &= \min(p(x), q(x)) \\
    \pi(a, x, x) &= \min(p(x) - \pi(x,a,x), Q(a,x)) \\
    \pi(a, x, a) &= \min(p(a), \sum_{x\in \mathcal{V}} (Q(a,x) - \pi(a,x,x)) \\
    &\quad \cdot \frac{Q(a,x) - \pi(a,x,x)}{\sum_{x\in \mathcal{V}} (Q(a,x) - \pi(a,x,x))}\\
    \pi(x, a, a) &= \min(p(a) - \sum_x \pi(a, x, a),\\
    &\quad\quad\quad \sum_{x\in\mathcal{V}} q(x) - \pi(x, a, x))\\
    &\cdot \frac{ q(x) - \pi(x, a, x)}{\sum_{x\in\mathcal{V}} q(x) - \pi(x, a, x)}
\end{align*}

Here, we provide the pseudocode for using SpecHub in real life. We follow a sequential procedure and avoid explicitly writing out the underlying coupling $\pi$. 

\begin{algorithm}[h]
 \caption{GetResidual}\label{alg:residual}
\begin{algorithmic}[1]
\STATE \textbf{Inputs:} target distribution $p$, draft distribution $q$, highest probability token $a$
\FORALL{$x$ in $\mathcal{V}, x \neq a$}
\STATE $p'(x) = \max \left(p(x)-q(x), 0\right)$
\STATE $q'(x) = \max \left(q(x)-p(x), 0\right)$
\ENDFOR
\STATE $p'(a) = p(a)$
\STATE $q'(a) = 0$
\RETURN $p'$, $q'$
\end{algorithmic}
\end{algorithm}

\begin{algorithm}[h]
 \caption{Sampling and Verification with SpecHub}\label{alg:SpecHub}
\begin{algorithmic}
\STATE \textbf{Inputs:} target distribution $p$, draft distribution $q$, vocabulary $\mathcal{V}$

\STATE Let $a = \arg\max_x q(x)$ be the token with the highest draft probability. 

\FORALL{$i \in \mathcal{V}$, $x \neq a$}
    \STATE $Q(x,a) = q(x)$, $Q(a,x) = \frac{q(a) q(x)}{1-q(a)}$
\ENDFOR
\STATE Sample draft tokens $x^{(1)}, x^{(2)} \sim Q$

\IF{$x^{(2)} = a$}
    \STATE \textbf{Return} $x^{(1)}$ with probability $\min \left(\frac{p(x^{(1)})}{Q(x^{(1)},a)}, 1 \right)$
\ENDIF
\STATE $p', Q'(*, a) = $\texttt{GetResidual}$(p, Q(*,a), a)$

\IF{$x^{(1)} = a$}
    \STATE \textbf{Return} $x^{(2)}$ with probability $\min \left(\frac{p'(x^{(2)})}{Q(a,x^{(2)})}, 1 \right)$
\ENDIF
\STATE $p'', Q'(a, *) = $\texttt{GetResidual}$(p', Q(a,*), a)$

\IF{$x^{(1)} = a$}
    \STATE \textbf{Return} $a$ with probability $\min \left(\frac{p(a)}{\sum_x Q'(a,x)}, 1 \right)$
    \STATE $p'(a) = \max(p(a) - \sum_x Q'(a,x), 0)$
\ENDIF

\IF{$x^{(2)} = a$}
    \STATE \textbf{Return} $a$ with probability $\min \left(\frac{p'(a)}{\sum_x Q'(x,a)}, 1 \right)$
    \STATE $p''(a) = \max(p'(a) - \sum_x Q'(x,a), 0)$
\ENDIF

\STATE \textbf{Return} a token sampled from the residual distribution $\text{norm}(p'')$
\end{algorithmic}
\end{algorithm}

\subsection{Correctness}

Here, we proof that SpecHub does not sacrifice the quality of generation. 

\begin{theorem}\label{thm:SpecHub_correctness}
Given a target distribution \( p \) and a draft distribution \( q \), SpecHub generates tokens such that for any token \( x \in \mathcal{V} \), the probability of generating \( x \) under SpecHub, denoted as \( \mathbb{P}(X=x) \), is equal to \( p(x) \). 
\end{theorem}
\begin{proof}
Given a target distribution \( p \) and a draft distribution \( q \), we need to show that SpecHub generates tokens such that for any token \( x \in \mathcal{V} \), the probability of generating \( x \) under SpecHub, denoted as \( P_{\text{SpecHub}}(x) \), is equal to \( p(x) \).

First, all draft pairs sampled by SpecHub involve the top token \( a = \arg\max_{x \in \mathcal{V}} q(x)\). For all $x \neq a$, pairs \( (x, a) \) and \( (a, x) \) are sampled with probabilities \( Q(x, a) = q(x) \) and \( Q(a, x) = \frac{q(a) q(x)}{1 - q(a)} \), respectively.

For a token \( x \neq a \), in the first draft, SpecHub generates \( x \) with probability 
\begin{align*}
    &\mathbb{P}(x = x^{(1)} \text{ and } X = x) \\
    =\ &Q(x,a) \min\left( \frac{p(x)}{Q(x, a)}, 1 \right) \\
    =\ &\min\left( p(x), q(x) \right).
\end{align*}

In the second draft, given that \( x \neq a \), the residual probability for token \( x \) after the first draft, denoted as \( p'(x) \), is:
\begin{align*}
    p'(x) = &\max(p(x) - q(x), 0)\\
    = &\ p(x) - \min (p(x), q(x))
\end{align*}

SpecHub generates \( x \) in the second draft with probability 
\begin{align*}
    &\mathbb{P}(x = x^{(2)} \text{ and } X = x) \\
    =\ &Q(a, x) \min\left( \frac{p'(x)}{Q(a, x)}, 1 \right) \\
    =\ &\min\left( p(x) - \min(p(x), q(x)), Q(a, x) \right) \\
    =\ &\min\left( p(x) - \min(p(x), q(x)), \frac{q(a) q(x)}{1 - q(a)} \right).
\end{align*}

Now, let's calculate the residual distribution after both drafts for tokens \( x \neq a \). The residual probability \( p''(x) \) for token \( x \) is calculated as follows:
\begin{align*}
&p''(x)  \\
=\ &\max(p'(x)-Q(a,x), 0) \\
=\ &\max\left(p(x) - q(x) - \frac{q(a) q(x)}{1 - q(a)}, 0\right)
\end{align*}

Since \( p''(x) \) represents the remaining probability after both drafts, it ensures that:
\begin{align*}
&\mathbb{P}(X = x) \\
=\ &\mathbb{P}(x = x^{(1)} \text{ and } X = x) \\
&+\mathbb{P}(x = x^{(2)} \text{ and } X = x) \\
&+p''(x) \\
=\  &\min(p(x), q(x)) \\
&+\min\left(p(x) - \min(p(x), q(x)), \frac{q(a) q(x)}{1 - q(a)}\right) \\
&+\max\left(p(x) - q(x) - \frac{q(a) q(x)}{1 - q(a)}, 0\right)\\
=\ &p(x)
\end{align*}

Now for \( x = a \):

In the first draft, SpecHub generates \( a \) with probability 
\begin{align*}
    &\mathbb{P}(a = x^{(1)} \text{ and } X = a) \\
    =\ &\sum_x Q'(a, x)  \min\left( \frac{p(a)}{\sum_x Q'(a, x)}, 1 \right) \\
    =\ &\min\left( p(a), \sum_x Q'(a, x) \right).
\end{align*}

In the second draft, given that \( a = x \), the residual probability for token \( a \) after the first draft, denoted as \( p'(a) \), is:
\begin{align*}
    p'(a) = \max(p(a) - \sum_x Q'(a, x), 0).
\end{align*}

SpecHub generates \( a \) with probability 
\begin{align*}
    &\mathbb{P}(a = x^{(2)} \text{ and } X = a) \\
    =\ &\sum_x Q'(x, a)\min\left( \frac{p'(a)}{\sum_x Q'(x, a)}, 1 \right) \\
    =\ &\min\left( \max(p(a) - \sum_x Q'(a, x), 0), \right.\\
    &\left. \sum_x Q'(x, a) \right)
\end{align*}

The total probability for generating \( a \) is:
\begin{align*}
&\mathbb{P}(X = a) \\
=\ &\mathbb{P}(a = x^{(1)} \text{ and } X = a) \\
&+\mathbb{P}(a = x^{(2)} \text{ and } X = a) \\
=\  &\min\left( p(a), \frac{p(a)}{\sum_x Q'(a, x)} \right) \\
&+\min\left( \max(p(a) - \sum_x Q'(a, x), 0), \right.\\
&\left. \frac{p(a)}{\sum_x Q'(x, a)} \right) \\
=\ &\min\left(p(a), \sum_x Q'(a,x)+Q'(x,a)\right)
\end{align*}

It can be shown that $p(a) < \sum_x Q'(a,x)+Q'(x,a)$. First, since $Q(a,a) = 0)$, we have \begin{align*}
    &\sum_x Q(a,x) + Q(x,a) \\
    =&\sum_{x\in \mathcal{V}\setminus\{a\}} q(x) + \frac{q(a)q(x)}{1-q(a)}\\
    =\ &1
\end{align*}
Also, we have $p(a) = 1 - \sum_{x\in \mathcal{V}\setminus\{a\}} p(x)$. 
Thus, 
\begin{align*}
    &\sum_{x\in\mathcal{V}\setminus\{a\}}Q'(a,x) + Q'(x,a)\\
    =\ &\sum_{x\in\mathcal{V}\setminus\{a\}} (\max(Q(a,x)-p(x), 0)\\
    & +\max(Q(x,a)-p'(x), 0))\\
    =\ &\sum_{x\in\mathcal{V}\setminus\{a\}} \max(Q(a,x)+Q(x,a)-p(x),0)\\
    \geq\ &\sum_{x\in\mathcal{V}\setminus\{a\}}Q(a,x)+Q(x,a)-p(x)\\
    =\ &\sum_{x\in\mathcal{V}\setminus\{a\}}Q(a,x)+Q(x,a) -\sum_{x\in\mathcal{V}\setminus\{a\}}p(x)\\
    =\ &1 - (1 - p(a)) = p(a)
\end{align*}

Thus, for any token \( x \in \mathcal{V} \), the probability of generating \( x \) under SpecHub is equal to \( p(x) \), ensuring that the output distribution matches the target distribution \( p \).
\end{proof}

As a corrolary of the last part of the proof, SpecHub accepts as much top token $a$ as $p(a)$. 
\begin{corollary}[Top Token Acceptance]
\label{lemma:top_token}
    Given a draft distribution \( q \) and a target distribution \( p \), let \( a = \arg\max_{x \in \mathcal{V}} q(x) \) denote the token with the highest draft probability. Then, SpecHub generates token \( a \) with probability \( p(a) \).
    
\end{corollary}

\subsection{Acceptance Rate}
\label{appendix:acceptance_rate}

We here prove a sufficient condition for SpecHub to run faster than RRS. 

\begin{theorem} [Superiority over RRS]
\label{thm:superiority_over_rrs}
    Let $\alpha = \sum_{x \in \mathcal{V}} \min(q(x), p(x))$ be the acceptance rate of the first draft. SpecHub has a higher acceptance rate in the second draft if $\frac{q(a)}{1-q(a)} > 1 - \alpha$.
\end{theorem}

\begin{proof}
First, by Lemma~\ref{lemma:top_token}, SpecHub generates the top token $a$ with probability $p(a)$. This maximizes the acceptance rate for $a$.
Next, we calculate the second draft acceptance rate for every other token $x \in \mathcal{V} \setminus \{a\}$.

For RRS, the acceptance rate for token $x$ in the first draft is $\min(p(x), q(x))$. The residual probability for token $x$ after the first draft, denoted as $r(x)$, is:
\[
p'(x) = \frac{p(x) - \min(p(x), q(x))}{1 - \alpha}
\]
where $\alpha = \sum_{x \in \mathcal{V}} \min(p(x), q(x))$ is the overall acceptance rate in the first draft.
The second draft acceptance rate for token $x$ under RRS is then:
\[
(1 - \alpha) \min\left(\frac{p(x) - \min(p(x), q(x))}{1 - \alpha}, q(x)\right)
\]
which simplifies to:
\[
\min\left(p(x) - \min(p(x), q(x)), (1 - \alpha) q(x)\right)
\]

For SpecHub, the second draft acceptance rate for token $x$ is:
\[
\min\left(p(x) - \min(p(x), q(x)), \frac{q(a)}{1 - q(a)} q(x)\right)
\]

Comparing these rates shows that SpecHub has a higher acceptance rate if $\frac{q(a)}{1 - q(a)} > 1 - \alpha$.
\end{proof}

In practice, this condition is usually satisfied. For example, if $\alpha = 0.5$, then as long as the top token has probability $q(a) > \frac{1}{3} = 0.333$, we guarantee acceleration. Meanwhile, since SpecHub accepts top tokens up to $p(a)$, the above sufficient conditions become necessary only in unusual cases when $p(a) = 0$. 

Using a similar proof strategy, we can show it guarantees to outperform OTM with independent sampling in rare cases. 

\begin{theorem} [Superiority over OTM]
\label{thm:superiority_over_OTM}
    SpecHub guarantees a higher total acceptance rate compared to OTM with independent sampling if $q(a) > 1/2$.
\end{theorem}

\begin{proof}
    Let $Q = q^{\otimes 2}$. Then, for a token $x$, the highest rate acceptance is upper bounded by the probability that it is contained in any draft pair with probability $1- (1-q(x))^2 < 2q(x)$. Meanwhile, for the first and second drafts, the acceptance rate when using SpecHub is $\frac{q(a)}{1 - q(a)} q(x) + q(x) = \frac{q(x)}{1 - q(a)}$.  Thus, we can accept more tokens $x$ if $\frac{q(x)}{1 - q(a)} > 2q(x)$, or $q(a) > \frac{1}{2}$.
\end{proof}

Compared to the previous theorem, this bound is nowhere near as tight since we are using a loose upper bound on OTM's performance. In reality we expect OTM to perform worse.

\subsection{First Draft Acceptance Rate}
\label{appendix:first_draft}

SpecHub is designed to optimize the acceptance rate across multiple drafts, but in rare cases, it might slightly decrease the acceptance rate of the top token in the first draft. This situation occurs when the probability of the top token in the target distribution satisfies $p(a)> q(a)$ and another token $x$ takes some of the probability mass $Q(a,x)$. However, our empirical evaluations in Table~\ref{tab:first_draft_acc} demonstrate that this effect is not noticeable in practice, as the acceptance rates of the first draft remain high.
\begin{table}[ht]

\centering
\caption{First Draft Acceptance Rates for SpecHub and RRSw across different models and temperatures.}
\begin{tabular}{cccc}

\toprule
$T$ & Draft & SpecHub & RRSw \\
\midrule
0.3 & JF68m & 0.4921 & 0.4498 \\
    & JF160m & 0.5578 & 0.5465 \\
0.6 & JF68m & 0.4842 & 0.4821 \\
    & JF160m & 0.5632 & 0.5587 \\
1   & JF68m & 0.4248 & 0.4418 \\
    & JF160m & 0.5130 & 0.5257 \\
\bottomrule
\end{tabular}
\label{tab:first_draft_acc}
\end{table}

\section{A discussion on more drafts}
\label{appendix:more_drafts}

\subsection{Diminishing Returns of Increasing Drafts}
While theoretically appealing, using more drafts in practice offers diminishing returns. As we increase the number of drafts, the probability mass of the residual distribution decreases, leading to lower acceptance rates for subsequent drafts. This phenomenon is illustrated in Figure \ref{fig:acc_T1}, where we present the acceptance rates for up to 10 drafts using both RRSw and RRS with temperature $T=1.0$. As evident from the plots, the acceptance rate drastically decreases after the first few drafts, suggesting that the benefit of using more than 5 drafts is negligible.
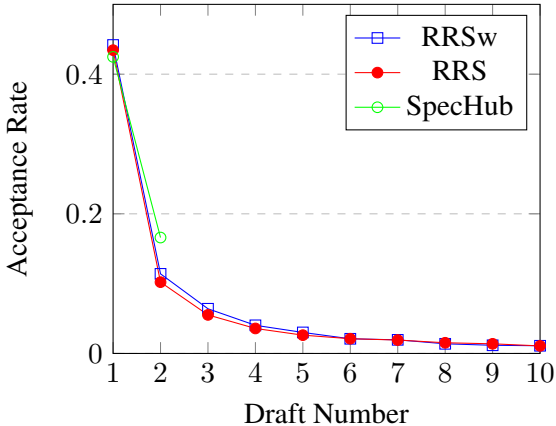
\begin{figure}[ht]
    \centering
    \begin{tikzpicture}
        \begin{axis}[
            xlabel=Draft Number,
            ylabel=Acceptance Rate,
            xmin=1, xmax=10, % Truncated x-axis
            ymin=0, ymax=0.5,
            xtick={1,2,3,...,10},
            legend pos=north east,
            ymajorgrids=true,
            grid style=dashed,
            width=0.45\textwidth,
        ]
        
        \addplot[
            color=blue,
            mark=square,
        ]
            coordinates {
                (1, 0.4418)(2, 0.1140)(3, 0.0640)(4, 0.0403)(5, 0.0301)
                (6, 0.0210)(7, 0.0195)(8, 0.0138)(9, 0.0117)(10, 0.0110)
            };
            \addlegendentry{RRSw}
            
        \addplot[
            color=red,
            mark=*,
        ]
            coordinates {
                (1, 0.4344)(2, 0.1021)(3, 0.0552)(4, 0.0358)(5, 0.0262)
                (6, 0.0211)(7, 0.0192)(8, 0.0153)(9, 0.0138)(10, 0.0108)
            };
            \addlegendentry{RRS}

        \addplot[ % Added SpecHub data
            color=green,
            mark=o,
        ]
            coordinates {
                (1, 0.4248)(2, 0.166)
            };
            \addlegendentry{SpecHub}
            
        \end{axis}
    \end{tikzpicture}
    \caption{Acceptance rate decay for different drafts with temperature $T=1.0$.}
    \label{fig:acc_T1}
\end{figure}
\subsection{Curse of Dimensionality}
The computational complexity of finding the optimal coupling in Multi-Draft Speculative Decoding grows exponentially with the number of drafts. This is often referred to as the curse of dimensionality. Specifically, the number of variables in the LP formulation is on the order of $O(|\mathcal{V}|^{k+1})$, where $|\mathcal{V}|$ is the vocabulary size and $k$ is the number of drafts. As $k$ increases, solving the LP becomes computationally intractable for even moderately sized vocabularies.

\begin{table*}[t]
    \centering
\caption{{\bf Acceptance Rates for Toy Experiments} The acceptance rates for SpecHub, Recursive Rejection Sampling (RRS), and Optimal Transport (OTM) algorithms using toy example drafts and target distributions. $T$ represents the temperature, and $\lambda$ controls the similarity between the draft and target distributions. We highlight the \textbf{best}, \underline{second best}, and \textit{third best} entries. }
    \label{tab:toy_experiments}
    \begin{tabular}{cc|ccccl}
        \toprule
        T & $\lambda$ & RRS & RRSw & OTM & OTMw & SpecHub \\
        \midrule
        0.1  & 0.7 & 0.6273 & \textit{0.7120} & 0.6380 & \underline{0.7345} & \textbf{0.7402} \\
        0.1  & 0.5 & 0.3323 & \textit{0.4057} & 0.3346 & \textbf{0.4125} & \underline{0.4123} \\
        0.25 & 0.7 & 0.7354 & 0.7653 & \textit{0.7846} & \textbf{0.8321} & \underline{0.8113} \\
        0.25 & 0.5 & 0.4564 & \underline{0.4978} & 0.4743 & \textbf{0.5245} & \textit{0.4968} \\
        0.5  & 0.7 & 0.8090 & 0.8122 & \underline{0.9037} & \textbf{0.9150} & \textit{0.8500} \\
        0.5  & 0.5 & 0.6456 & \textit{0.6593} & \underline{0.7052} & \textbf{0.7206} & 0.6403 \\
        \bottomrule
    \end{tabular}
\end{table*}
\subsection{Potential for Sparse Algorithms on more drafts}
The diminishing returns of additional drafts and the curse of dimensionality suggest that a practical approach should focus on a small number of drafts while ensuring an efficient probability of mass transport. One promising direction is to explore sparse algorithms that leverage the structure of the draft and target distributions. For instance, instead of considering all possible combinations of drafts, we can prioritize those with higher sampling probabilities or those that exhibit significant overlap between the draft and target distributions.
One potential approach is to extend the "hub" concept of SpecHub to multiple drafts. Instead of designating a single token as the hub, we can identify a small set of high-probability tokens and create a sparse flow network where probability mass is primarily transported through these hubs. This approach could potentially maintain high acceptance rates while significantly reducing the computational complexity compared to solving the full LP. Further research in this direction could lead to more efficient and scalable algorithms for MDSD.

\section{Comparing SpecHub to OTM in Toy Settings}

Here, we seek to compare OTM, RRS, and SpecHub's performance by measuring the acceptance rate of the three algorithms using a few toy example drafts and target distributions with a small vocab size $|\mathcal{V}| = 50$ in Table~\ref{tab:toy_experiments}. Given temperature $T$ and a hyperparameter $\lambda$ that controls the similarity between the two distributions, we generate two logits using uniform distributions such that $u_p \sim \text{Unif}(0,1)^{\otimes |\mathcal{V}|}$ and $u_q \sim \text{Unif}(0,1)^{\otimes |\mathcal{V}|}$. The corresponding target and draft distributions are $p = \text{softmax}(\frac{u_p}{T})$ and $q = \text{softmax}(\lambda \frac{u_p}{T} + (1-\lambda) \frac{u_q}{T})$. We calculate the acceptance rate for all methods theoretically except for RRS without replacement, where we perform a Monte-Carlo Simulation with a thousand repetitions. We conduct the experiment on a hundred pairs of toy distributions and report the average. 
The results in Table~\ref{tab:toy_experiments} quantitatively illustrate the performance differences among SpecHub, Recursive Rejection Sampling (RRS), RRS without replacement, and Optimal Transport (OTM) methodologies under varying conditions of temperature $T$ and similarity parameter $\lambda$. In high similarity scenarios ($\lambda = 0.7$), SpecHub outperforms other methods significantly at lower temperatures ($T=0.1$), achieving the best acceptance rate of \textbf{0.7402}, closely followed by OTM without replacement at \underline{0.7345}. At higher temperatures ($T=0.5$), OTM methods, particularly OTM without replacement, dominate, marking the best performance with \textbf{0.9150} at $T=0.5$ and $\lambda=0.7$. This suggests that SpecHub is particularly effective in tightly controlled environments with high similarity between distributions and low entropy, whereas OTM shines with increased distribution complexity. SpecHub's consistent performance across different conditions emphasizes its robustness, particularly when distribution similarity is moderate ($\lambda = 0.5$), where it maintains competitive acceptance rates, closely trailing the best results.

\section{Maximum Flow Problem Formulation}
\label{subsec:max_flow}

At $k=2$, our Linear Programming (LP) formulation describes an equivalent Maximum Flow Problem formulation. This formulation effectively models the Multi-Draft Speculative Decoding process as the transportation of probability mass through a network of pipes.

Given an LP formulation with vocabulary set $\mathcal{V}$, pair sampling distribution $Q\in \Delta^{|\mathcal{V}|^2-1}$, and target distribution $p\in \Delta^{|\mathcal{V}|-1}$, we construct a graph $G = (V, E)$ where the vertex set $V$ consists of the vocabulary $\mathcal{V}$, a source vertex $s$, and a sink vertex $t$. The capacity function $g: (u,v) \in E \to [0,1]$ is defined for each edge as follows:
\begin{align*}
g(u,v) = \begin{cases}
    \sum_{x_2} Q_{vx_2}, & \text{if } u = s \text{ and } v \in \mathcal{V},\\
    p(v), & \text{if } u \in \mathcal{V} \text{ and } v = t,\\
    Q_{uv}, & \text{if } u, v\in \mathcal{V} \text{ and } u \neq v,\\
    0, & \text{otherwise}.
\end{cases}
\end{align*}
In this formulation, the source vertex $s$ distributes the total probability mass to the vertices in the vocabulary set $\mathcal{V}$, while the sink vertex $t$ collects the transported probability mass from the vocabulary vertices. The edges between the vocabulary vertices represent the possible transitions dictated by the pair sampling distribution $Q$. This network flow model not only provides an intuitive visualization of the probability mass transport process but also allows us to leverage well-established algorithms in network flow theory to solve the MDSD problem efficiently.

\section{More Experiment Details}
\label{appendix:experiments}

\paragraph{JF68m on Full Binary Trees and Binary Sequoia Unbalanced Trees}
We conducted experiments to measure the batch efficiency of the JF68m model on both full binary trees and binary Sequoia unbalanced trees. For the full binary trees, we tested tree depths ranging from $d=2$ to $d=5$, and for the binary Sequoia trees, we used an unbalanced tree structure with varying depths. The results demonstrate that SpecHub consistently outperforms both RRS and RRSw across all tree depths. In the full binary tree configuration, SpecHub achieves a batch efficiency improvement of $0.02-0.10$ over RRSw and $0.04-0.20$ over RRS at temperatures $T=0.6$ and $1.0$. For the binary Sequoia unbalanced trees, SpecHub maintains a higher batch efficiency, confirming its robustness on the more efficient unbalanced tree structures. 

\paragraph{JF160m on Binary and Ternary Trees}
We also evaluated the batch efficiency of RRS and RRSw using the JF160m model on both binary and ternary trees. For binary trees, we tested tree depths from $d=2$ to $d=6$, and for ternary trees, we considered depths up to $d=4$. The JF160m model shows significant improvements in batch efficiency when using SpecHub. At temperatures $T=0.6$ and $1.0$, SpecHub outperforms RRS by $0.03-0.12$ and RRSw by $0.05-0.15$ in binary tree configurations. The performance of RRS and RRSw in the ternary tree setting is worse than SpecHub on binary trees, suggesting the benefit of using more drafts is less significant. 

\paragraph{EAGLE Decoding Head}
To further explore the efficiency of our proposed method, we implemented the SpecHub algorithm using the EAGLE decoding head. The batch efficiency was evaluated on binary trees of depths $d=2$ to $d=5$. SpecHub with the EAGLE decoding head shows a substantial increase in efficiency, generating up to $3.53$ and $3.33$ tokens per iteration at temperatures $T=0.6$ and $1.0$, respectively. This represents an additional $0.08$ tokens per iteration compared to RRS without replacement. The experimental results reinforce the benefits of integrating SpecHub with advanced decoding heads like EAGLE, particularly in enhancing batch efficiency.

\paragraph{Larger Models} 
In addition to the results reported in the main paper, we conducted further experiments on larger models, specifically Llama 2-13B-Chat and Vicuna-1.3-33B, to evaluate the scalability of SpecHub. As shown in Tables~\ref{tab:batch-efficiency-llama13b} and \ref{tab:batch-efficiency-vicuna33b}, SpecHub consistently outperforms both Recursive Rejection Sampling (RRS) and RRS without replacement across all configurations of tree depth and temperature ($T=0.6$ and $T=1.0$). For Llama 2-13B-Chat, SpecHub achieves up to 2.77 tokens per step on the OpenWebText dataset, while for Vicuna-33B, it generates up to 3.32 tokens per step on the MT-Bench dataset. These results highlight SpecHub's ability to maintain high batch efficiency and token acceptance rates as model size increases, demonstrating its robust scalability when applied to larger language models.

\begin{table*}
\caption{{\bf Batch Efficiency Results for JF68m Data} Average accepted tokens and batch efficiency for different configurations of target model and draft model pairs across various temperatures. SpecHub consistently outperforms RRS and RRSw in both acceptance rate and batch efficiency. We also include binary Sequoia trees and show that SpecHub performs well on unbalanced trees.}
\label{tab:batch-efficiency-68}
\centering
\begin{tabular}{c|c|c|c|c|c|c|c|c|c}
\hline
T & Data & Tree & RRS & RRSw & SpecHub & Tree & RRS & RRSw & SpecHub \\
\hline
0.6 & CNN & $2^2$ & 1.5540 & 1.5997 & \bfseries 1.6157 & biSeq4 & 1.7938 & 1.8304 & \bfseries 1.8498 \\
0.6 & OWT & $2^2$ & 1.5485 & 1.5895 & \bfseries 1.6080 & biSeq4 & 1.7971 & 1.8225 & \bfseries 1.8424 \\
0.6 & CNN & $2^3$ & 1.8482 & 1.9685 & \bfseries 1.9863 & biSeq8 & 2.0361 & 2.1540 & \bfseries 2.1542 \\
0.6 & OWT & $2^3$ & 1.8576 & 1.9241 & \bfseries 1.9632 & biSeq8 & 2.0247 & 2.1005 & \bfseries 2.1285 \\
0.6 & CNN & $2^4$ & 2.0510 & 2.1694 & \bfseries 2.2456 & biSeq16 & 2.1354 & 2.2667 & \bfseries 2.2839 \\
0.6 & OWT & $2^4$ & 2.0256 & 2.1299 & \bfseries 2.2103 & biSeq16 & 2.1378 & \bfseries 2.2153 & 2.2064 \\
0.6 & CNN & $2^5$ & 2.1385 & 2.3149 & \bfseries 2.4031 & biSeq32 & 2.2452 & 2.4198 & \bfseries 2.4353 \\
0.6 & OWT & $2^5$ & 2.0867 & 2.2295 & \bfseries 2.3416 & biSeq32 & 2.2007 & 2.3556 & \bfseries 2.3868 \\
\hline
1.0 & CNN & $2^2$ & 1.5432 & 1.5521 & \bfseries 1.5997 & biSeq4 & 1.7401 & 1.7469 & \bfseries 1.8057 \\
1.0 & OWT & $2^2$ & 1.5488 & 1.5509 & \bfseries 1.5905 & biSeq4 & 1.7355 & 1.7437 & \bfseries 1.7879 \\
1.0 & CNN & $2^3$ & 1.8384 & 1.8790 & \bfseries 1.9832 & biSeq8 & 1.9522 & 2.0063 & \bfseries 2.0667 \\
1.0 & OWT & $2^3$ & 1.8232 & 1.8585 & \bfseries 1.9661 & biSeq8 & 1.9304 & 2.0008 & \bfseries 2.0720 \\
1.0 & CNN & $2^4$ & 1.9762 & 2.0441 & \bfseries 2.2106 & biSeq16 & 2.0529 & 2.1662 & \bfseries 2.2843 \\
1.0 & OWT & $2^4$ & 1.9954 & 2.0493 & \bfseries 2.1957 & biSeq16 & 2.0330 & 2.1030 & \bfseries 2.2619 \\
1.0 & CNN & $2^5$ & 2.0694 & 2.1383 & \bfseries 2.3104 & biSeq32 & 2.1197 & 2.1604 & \bfseries 2.3445 \\
1.0 & OWT & $2^5$ & 2.0890 & 2.1574 & \bfseries 2.3149 & biSeq32 & 2.1008 & 2.1950 & \bfseries 2.3571 \\
\hline
\end{tabular}
\end{table*}

\begin{table*}
\caption{{\bf Batch Efficiency Results for JF160m Data} Average accepted tokens and batch efficiency for different configurations of target model and draft model pairs at $T=0.6$ and $T=1.0$. The results are presented for CNN and OpenWebText datasets, comparing RRS, RRS without replacement, and TransportHub. We also contained results on ternary trees to showcase that using $k>2$ gives diminishing gain. }
\label{tab:batch-efficiency-160}
\centering
\begin{tabular}{c|c|c|c|c|c}
\hline
T & Data & Tree & RRS & RRS w/o & SpecHub \\
\hline
0.6 & CNN & $2^2$ & 1.633994691 & 1.667634674 & \bfseries 1.6861 \\
0.6 & OpenWebText & $2^2$ & 1.641550493 & 1.672971282 & \bfseries 1.677 \\
0.6 & CNN & $2^3$ & 2.016376307 & 2.142804292 & \bfseries 2.1758 \\
0.6 & OpenWebText & $2^3$ & 2.052868003 & 2.113952048 & \bfseries 2.115 \\
0.6 & CNN & $3^2$ & 1.66262118 & 1.734944266 & \diagbox{}{} \\ 
0.6 & OpenWebText & $3^2$ & 1.669826224 & 1.70473377 & \diagbox{}{} \\
0.6 & CNN & $2^4$ & 2.282944241 & 2.369522017 & \bfseries 2.4841 \\
0.6 & OpenWebText & $2^4$ & 2.28490566 & 2.411659014 & \bfseries 2.4492 \\
0.6 & CNN & $3^3$ & 2.113219655 & 2.279599835 & \diagbox{}{} \\
0.6 & OpenWebText & $3^3$ & 2.111602497 & 2.212962963 & \diagbox{}{} \\
0.6 & CNN & $2^5$ & 2.378323523 & 2.604486152 & \bfseries 2.7238 \\
0.6 & OpenWebText & $2^5$ & 2.449243411 & 2.642651616 & \bfseries 2.6901 \\
0.6 & CNN & $3^4$ & 2.39760652 & 2.681949084 & \diagbox{}{} \\
0.6 & OpenWebText & $3^4$ & 2.433582166 & 2.667044296 & \diagbox{}{} \\
\hline
1.0 & CNN & $2^2$ & 1.608515798 & 1.633187465 & \bfseries 1.6748 \\
1.0 & OpenWebText & $2^2$ & 1.633351663 & 1.635781207 & \bfseries 1.6834 \\
1.0 & CNN & $2^3$ & 1.959878368 & 2.053886546 & \bfseries 2.1362 \\
1.0 & OpenWebText & $2^3$ & 2.028797337 & 2.077786547 & \bfseries 2.1584 \\
1.0 & CNN & $3^2$ & 1.663016602 & 1.689861121 & \diagbox{}{} \\
1.0 & OpenWebText & $3^2$ & 1.677094972 & 1.701585742 & \diagbox{}{}\\
1.0 & CNN & $2^4$ & 2.20357984 & 2.286009649 & \bfseries 2.4204 \\
1.0 & OpenWebText & $2^4$ & 2.295532975 & 2.379759419 & \bfseries 2.4922 \\
1.0 & CNN & $3^3$ & 2.105012354 & 2.165854573 & \diagbox{}{} \\
1.0 & OpenWebText & $3^3$ & 2.166307084 & 2.233691623 & \diagbox{}{} \\
1.0 & CNN & $2^5$ & 2.315296164 & 2.41812897 & \bfseries 2.6624 \\
1.0 & OpenWebText & $2^5$ & 2.429887821 & 2.532017591 & \bfseries 2.7334 \\
1.0 & CNN & $3^4$ & 2.382244389 & 2.474047719 & \diagbox{}{} \\
1.0 & OpenWebText & $3^4$ & 2.467950678 & 2.550284031 & \diagbox{}{} \\
\hline
\end{tabular}
\end{table*}

\begin{table}[h]
\caption{{\bf Batch Efficiency Results for SpecHub and RRS using EAGLE} The batch efficiency of SpecHub and Recursive Rejection Sampling (RRS) methods when applied with EAGLE. The table reports average accepted tokens per step across different temperatures and datasets, demonstrating that SpecHub consistently outperforms RRS.}
\label{tab:eagle_spechub_rrs_results}
\centering
\begin{tabular}{cccccc}
\toprule
T & Tree & RRS & RRS-wo & SpecHub \\
\midrule
    0.6 & $2^2$ & 1.8211 & 1.8687 & \textbf{1.8825} \\
    0.6 & $2^3$ & 2.4325 & 2.5585 & \textbf{2.5939} \\
    0.6 & $2^4$ & 2.9125 & 3.0899 & \textbf{3.1192} \\
    0.6 & $2^5$ & 3.2501 & 3.4838 & \textbf{3.5380} \\
    1.0 & $2^2$ & 1.8054 & 1.8327 & \textbf{1.8655} \\
    1.0 & $2^3$ & 2.3961 & 2.4737 & \textbf{2.4850} \\
    1.0 & $2^4$ & 2.8425 & 2.9019 & \textbf{3.0281} \\
    1.0 & $2^5$ & 3.1451 & 3.2548 & \textbf{3.3318} \\
\bottomrule
\end{tabular}

\label{tab:eagle_batch_efficiency}
\end{table}

\begin{table*}[h]
\caption{{\bf Batch Efficiency Results for Llama 2-13B-Chat} Average accepted tokens and batch efficiency for Llama 2-13B-Chat with different configurations at $T=0.6$ and $T=1.0$. The results compare RRS, RRS without replacement, and SpecHub on the CNN and OpenWebText datasets.}
\label{tab:batch-efficiency-llama13b}
\centering
\begin{tabular}{c|c|c|c|c|c}
\hline
T & Data & Tree & RRS & RRS w/o & SpecHub \\
\hline
0.6 & CNN & $2^2$ & 1.5876 & 1.6420 & \bfseries 1.6628 \\
0.6 & OpenWebText & $2^2$ & 1.6307 & 1.6724 & \bfseries 1.6958 \\
0.6 & CNN & $2^3$ & 1.9981 & 2.0507 & \bfseries 2.0888 \\
0.6 & OpenWebText & $2^3$ & 2.0151 & 2.1424 & \bfseries 2.1798 \\
0.6 & CNN & $2^4$ & 2.1921 & 2.3305 & \bfseries 2.3873 \\
0.6 & OpenWebText & $2^4$ & 2.3048 & 2.4331 & \bfseries 2.5251 \\
0.6 & CNN & $2^5$ & 2.3181 & 2.4716 & \bfseries 2.5689 \\
0.6 & OpenWebText & $2^5$ & 2.4882 & 2.6769 & \bfseries 2.7774 \\
\hline
1.0 & CNN & $2^2$ & 1.5993 & 1.6251 & \bfseries 1.6542 \\
1.0 & OpenWebText & $2^2$ & 1.6192 & 1.6415 & \bfseries 1.6687 \\
1.0 & CNN & $2^3$ & 1.9477 & 1.9839 & \bfseries 2.0878 \\
1.0 & OpenWebText & $2^3$ & 1.9994 & 2.0716 & \bfseries 2.1538 \\
1.0 & CNN & $2^4$ & 2.1784 & 2.2354 & \bfseries 2.3739 \\
1.0 & OpenWebText & $2^4$ & 2.2323 & 2.3355 & \bfseries 2.4581 \\
1.0 & CNN & $2^5$ & 2.2778 & 2.3684 & \bfseries 2.5540 \\
1.0 & OpenWebText & $2^5$ & 2.4226 & 2.5218 & \bfseries 2.7190 \\
\hline
\end{tabular}
\end{table*}

\begin{table*}[h]
\caption{{\bf Batch Efficiency Results for Vicuna-1.3-33B} Average accepted tokens and batch efficiency for Vicuna-33B with different configurations at $T=0.6$ and $T=1.0$. The results compare RRS, RRS without replacement, and SpecHub on the MT-Bench dataset.}
\label{tab:batch-efficiency-vicuna33b}
\centering
\begin{tabular}{c|c|c|c|c|c}
\hline
T & Data & Tree & RRS & RRS w/o & SpecHub \\
\hline
0.6 & MT-Bench & $2^2$ & 1.7878 & 1.8327 & \bfseries 1.8490 \\
0.6 & MT-Bench & $2^3$ & 2.3469 & 2.4694 & \bfseries 2.5098 \\
0.6 & MT-Bench & $2^4$ & 2.7273 & 2.9298 & \bfseries 2.9762 \\
0.6 & MT-Bench & $2^5$ & 3.0102 & 3.2114 & \bfseries 3.3245 \\
\hline
1.0 & MT-Bench & $2^2$ & 1.7792 & 1.8257 & \bfseries 1.8454 \\
1.0 & MT-Bench & $2^3$ & 2.3399 & 2.4169 & \bfseries 2.4650 \\
1.0 & MT-Bench & $2^4$ & 2.7252 & 2.8649 & \bfseries 2.9230 \\
1.0 & MT-Bench & $2^5$ & 3.0190 & 3.1116 & \bfseries 3.2079 \\
\hline
\end{tabular}
\end{table*}

\end{document}